\documentclass{article} 
\usepackage{collas2023_conference,times}

\usepackage{mymath}
\usepackage{enumitem}

\usepackage{hyperref}
\hypersetup{
    colorlinks=true,
    linkcolor=red,
    filecolor=magenta,
    urlcolor=blue,
    citecolor=purple,
    pdftitle={Overleaf Example},
    pdfpagemode=FullScreen,
}

\allowdisplaybreaks

\title{Fixed Design Analysis of Regularization-Based Continual Learning}

\author{Haoran Li 
\thanks{Equal Contribution
}\\
Department of Computer Science\\
Rice University\\
Houston, TX \\
\texttt{lihr@rice.edu} \\
\And 
Jingfeng Wu $^*$  \\
Department of Computer Science \\
Johns Hopkins University \\
Baltimore, MD \\
\texttt{uuujf@jhu.edu} \\
\And 
Vladimir Braverman \\
Department of Computer Science\\
Rice University\\
Houston, TX \\
\texttt{vb21@rice.edu} \\
}

\collasfinalcopy 

\begin{document}

\maketitle

\begin{abstract}
We consider a continual learning (CL) problem with two linear regression tasks in the fixed design setting, where the feature vectors are assumed fixed and the labels are assumed to be random variables.
We consider an $\ell_2$-regularized CL algorithm, which computes an Ordinary Least Squares parameter to fit the first dataset, then computes another parameter that fits the second dataset under an $\ell_2$-regularization penalizing its deviation from the first parameter, and outputs the second parameter.
For this algorithm, we provide tight bounds on the average risk over the two tasks.
Our risk bounds reveal a provable trade-off between forgetting and intransigence of the $\ell_2$-regularized CL algorithm: with a large regularization parameter, the algorithm output forgets less information about the first task but is intransigent to extract new information from the second task; and vice versa. 
Our results suggest that catastrophic forgetting could happen for CL with dissimilar tasks (under a precise similarity measurement) 
and that a well-tuned $\ell_2$-regularization can partially mitigate this issue by introducing intransigence.
\end{abstract}

\section{Introduction}

In \emph{continual learning} (CL, also known as lifelong learning), an agent is provided with multiple learning tasks in a \emph{sequential} manner and aims to solve these tasks with \emph{limited long-term memory}.
Due to the shortage of long-term memory, a CL problem is fundamentally more challenging than a single-task learning problem, as the CL agent is unable to memorize all data (neither can the agent memorize multiple task-specific solutions respectively) and thus a CL problem cannot be reduced to a single-task (i.e., the combination of all tasks) learning problem \citep{parisi2019continual}. 
On the other hand, without using exceedingly large long-term memory, one may view a CL problem as an online multi-task problem and fits a model with the data provided for each task, sequentially. 
However, the final model outputted by such an online approach could be poor for CL, as the performance is evaluated with respect to all seen tasks, but the final model tends to forget information from earlier tasks.
This issue is known as \emph{catastrophic forgetting} \citep{mccloskey1989catastrophic}.

One effective category of CL methods mitigates catastrophic forgetting by \emph{regularization}.
Specifically, the CL agent stores a copy of the latest model in the long-term memory, and when fitting a new model for an incoming task, a regularization is imposed to prevent the model from deviating too much from the old one
\citep{kirkpatrick2017overcoming, aljundi2018memory, zenke2017continual}.
The regularization-based methods temper catastrophic forgetting by lowering its flexibility of fitting new data (aka, \emph{intransigence}): with a strong regularization strength, the agent tends to memorize old tasks more and learn the new task less; and vice versa for a weaker regularization strength. 
When the regularization strength is properly set, the regularization-based methods can learn new tasks well without forgetting old tasks catastrophically, hence achieving good performance for CL.
Understanding the \emph{trade-off} between \emph{forgetting} with \emph{intransigence} is the key to understanding the effectiveness of the regularization-based CL methods.   

\paragraph{Contributions.}
This work theoretically investigates the trade-off between forgetting and intransigence induced by a regularization-based CL algorithm.
Specifically, 
we consider a (domain-incremental) CL problem with two linear regression tasks in the fix design setting (see Assumption \ref{assump:fixed-design}). 
We consider an $\ell_2$-regularized CL algorithm ($\ell_2$-RCL), where an $\ell_2$-regularization is imposed when learning the second task, penalizing the model deviating from the one learned from the first task (see \eqref{eqn:regularized-learning}).
We make the following contributions:
\begin{enumerate}[leftmargin=*]
    \item We provide sharp risk bounds for $\ell_2$-RCL. This implies a provable trade-off between forgetting and intransigence, which is controlled by the regularization parameter: a strong regularization effectively reduces forgetting but unavoidably enlarges intransigence and vice versa. 
    \item We show that, without regularization, forgetting could become catastrophic when the two tasks are dissimilar (under a precise task similarity measurement). Moreover, for CL with moderately dissimilar tasks, the catastrophic forgetting can be tempered with a well-tuned $\ell_2$-regularization by introducing a tolerable intransigence.
    \item Finally, we show that there exists a CL problem with extremely dissimilar tasks that cannot be solved by $\ell_2$-RCL. In this case, forgetting and intransigence are in strong tension and cannot be reduced simultaneously. This indicates the necessity of considering a more advanced regularization for CL. 
\end{enumerate}
Our theoretical results are validated with numerical simulations. All proofs are deferred to the appendix.

\section{Preliminaries}

\paragraph{Two linear regression tasks.}
We set up a 2-task CL problem with two linear regression problems -- our results can be easily extended to CL with more than two tasks.
Let 
$\big( \xB^{(1)}_t, y^{(1)}_t\big)_{t=1}^n \in \Rbb^d \times \Rbb$ be $n$ pair of data vector and label variable that are drawn independently from a population distribution $\Dcal^{(1)}$, where $d>0$ refers to the ambient dimension.
Similarly, let $\big( \xB^{(2)}_t, y^{(2)}_t\big)_{t=1}^n \in \Rbb^d \times \Rbb$ be another $n$ pair of data vector and label variable that are drawn independently from a different population distribution $\Dcal^{(2)}$.
For simplicity, we denote  
\[
\XB^{(k)} := \big(\xB^{(k)}_1,\dots, \xB^{(k)}_n\big)^\top\in \Rbb^{n\times d},\quad \yB^{(k)} := \big(y^{(k)}_1,\dots,y^{(k)}_n \big)\in\Rbb^n,
\]
for $k=1,2$.
We use $\wB\in\Rbb^d$ to denote a model parameter. 
Then we define the population risks for the two linear regression tasks by 
\begin{align*}
    \Rcal_1(\wB) := \frac{1}{n} \Ebb_{\Dcal^{(1)}} \| \yB^{(1)}-\XB^{(1)}\wB \|_2^2,\qquad
    \Rcal_2(\wB) := \frac{1}{n} \Ebb_{ \Dcal^{(2)}} \|\yB^{(2)}-\XB^{(2)}\wB\|_2^2,
\end{align*}
respectively.

\paragraph{Continual learning.}
The goal of CL is to learn a parameter to minimize the \emph{joint population risk}
\begin{equation}\label{eqn:risk}
    \Rcal(\wB)  :=  \Rcal_1(\wB) + \Rcal_2(\wB).
\end{equation}
Unlike other multi-task learning problems, in CL the agent learns the parameter under the constraints that the task-specific datasets can only be accessed sequentially and that its long-term memory is limited. 
Specifically, a two-task CL problem involves two local learning phases and a memory consolidation phase \citep{kirkpatrick2017overcoming, zenke2017continual}.
A CL algorithm begins with the first local learning phase, where 
the CL agent is provided with the first dataset $\big(\XB^{(1)}, \yB^{(1)}\big)$ and is allowed to do arbitrary computations with the dataset.
After that, the agent is switched to the memory consolidation phase and must delete all of its local data except for a $d$-dimensional vector, which we call \emph{long-term memory}.
Then the agent enters the second local learning phase, where it retains the long-term memory and is provided with the second dataset $\big(\XB^{(2)}, \yB^{(2)}\big)$, and is allowed to do arbitrary computation with the $d$-dimensional long-term memory vector and the dataset.
The aim of the CL agent is to report a model (parameter) that achieves a small joint population risk \eqref{eqn:risk}.

\begin{remark}[Memory size]
In the above CL setup, we only allow a $d$-dimensional long-term memory vector.
On the one hand, this amount of memory is necessary for storing a $d$-dimensional model parameter (without additional knowledge). 
On the other hand, the condition can be relaxed to allow a $kd$-dimensional long-term memory vector for $k>1$, which enables more powerful CL algorithms (e.g., EWC \citep{kirkpatrick2017overcoming}, SCP \citep{kolouri2020sliced}, sketched EWC \citep{li2021lifelong}, etc.).
As this paper initiates the theoretical framework for CL, we focus on the simpler setting and leave the general cases as future works. 
\end{remark}

We now introduce two CL methods, called by us \emph{ordinary continual learning} (OCL) and \emph{$\ell_2$-regularized continual learning} ($\ell_2$-RCL), respectively.

\paragraph{Ordinary continual learning.}
In the first local learning phase,
the \emph{ordinary continual learning} (OCL) performs \emph{ordinary least square} \citep{bartlett2020benign} with the first dataset. 
In the memory consolidation phase, it saves the obtained model parameter (an $d$-dimensional vector, denoted by $\wB^{(1)}$) as its long-term memory.
Then in the second local learning phase, it computes a model parameter (denoted by $\wB^{(2)}$) that fits the second dataset while minimizing the $\ell_2$-deviation from the previous model parameter (i.e., the long-term memory vector $\wB^{(1)}$).
Specifically, OCL outputs $\wB^{(2)}$ such that
\begin{equation}\label{eqn:sequential-learning}
    \wB^{(1)} = \big({\XB^{(1)}}^{\top}\XB^{(1)}\big)^{-1} {\XB^{(1)}}^\top \yB^{(1)}; \qquad
    \wB^{(2)} = \wB^{(1)} + \big({\XB^{(2)}}^{\top}\XB^{(2)}\big)^{-1} {\XB^{(2)}}^\top \big(\yB^{(2)} - \XB^{(2)}\wB^{(1)} \big).
\end{equation}
It is worth noting that OCL has been considered in \citet{evron2022catastrophic}. However, their focus is on the optimization behaviors of OCL, while we focus on its statistical behaviors. 

\begin{remark}[OCL in the interpolation regime]
When the number of parameters exceeds the number of samples, i.e., $d>2n$, then both datasets can be interpolated almost surely  \citep{bartlett2020benign}. 
In the interpolation regime, \eqref{eqn:sequential-learning} corresponds to 
\begin{equation*}
    \wB^{(1)} = \argmin_{\wB: \yB^{(1)} = \XB^{(1)} \wB} \| \wB \|_2^2; \qquad
    \wB^{(2)} = \argmin_{\wB: \yB^{(2)} = \XB^{(2)} \wB} \| \wB - \wB^{(1)} \|_2^2.
\end{equation*}
Our results make no assumption on $d$ and can be applied in both underparameterized and overparameterized cases. 
\end{remark}

\paragraph{$\ell_2$-Regularized continual learning.}
The $\ell_2$-regularized continual learning ($\ell_2$-RCL) is identical to OCL in the first local learning phase and the memory consolidation phase. 
In the second local learning phase, however, the agent computes a model parameter that fits the second dataset under an $\ell_2$-penalty from deviating from the previous model parameter.  
Specially, $\ell_2$-RCL outputs $\wB^{(2)}$ such that
\begin{equation}\label{eqn:regularized-learning}
    \wB^{(1)} = \big({\XB^{(1)}}^{\top}\XB^{(1)}\big)^{-1} {\XB^{(1)}}^\top \yB^{(1)}; \qquad
    \wB^{(2)} = \argmin_{\wB}  \frac{1}{n} \| \yB^{(2)} - \XB^{(2)} \wB \|_2^2 + \mu \| \wB - \wB^{(1)} \|_2^2,
\end{equation}
where $\mu>0$ is a hyperparameter.
The $\ell_2$-RCL has been studied by \citet{heckel2022provable}.
However, they focus on the existence of continual learning settings that the $\ell_2$-RCL training can provably succeed or fail, while we focus on the success and failure condition of the $\ell_2$-RCL algorithm.
We also note that the $\ell_2$-RCL algorithm is a special case of the EWC algorithm \citep{kirkpatrick2017overcoming} when the covariance matrix of the first dataset is isotropic (i.e., all eigenvalues are equal).
Finally, we note that the $\ell_2$-RCL output reduces to OCL when $\mu\to0$.

\paragraph{A forgetting-intransigence decomposition.}
A CL algorithm targets to minimize a joint population risk (e.g., \eqref{eqn:risk} in our setting). 
Additionally, in practice, the performance of a CL algorithm is often measured by the following two metrics: (1) \emph{forgetting}, which refers to the amount of performance degradation with respect to the old tasks after learning a new task, and (2) \emph{intransigence}, which evaluates the (in)ability of the algorithm to fit a new task after having already fitted a sequence of old tasks \citep{lopez2017gradient,chaudhry2018riemannian}.
Heuristically, forgetting is believed to be in tension with intransigence, and a good CL algorithm achieves a balance between forgetting and intransigence \citep{lopez2017gradient,chaudhry2018riemannian}. 

Our main results theoretically characterize the trade-off between forgetting and intransigence for the $\ell_2$-RCL algorithm.
Specifically, let $\wB^{(2)}$ and $\wB^{(1)}$ be given by \eqref{eqn:regularized-learning}, then we define forgetting and 
intransigence by 
\begin{equation}\label{eqn:F+I}
\begin{aligned}
    \Fcal\big(\wB^{(2)}, \wB^{(1)}\big) 
    &:= \Rcal_1\big(\wB^{(2)}\big) - \Rcal_1\big(\wB^{(1)}\big),\\ 
    \Ical\big(\wB^{(2)}, \wB^{(1)}\big) 
     &:= \Rcal_2\big(\wB^{(2)}\big) - \min \Rcal_2 (\cdot) + \Rcal_1\big(\wB^{(1)}\big) - \min \Rcal_1(\cdot),
\end{aligned}
\end{equation}
respectively.
The following \emph{forgetting-intransigence decomposition} holds by a direct calculation:
\begin{equation}\label{eqn:risk=F+I}
\Rcal \big( \wB^{(2)} \big) -\min \Rcal (\cdot )
= \Fcal\big(\wB^{(2)}, \wB^{(1)}\big) + \Ical\big(\wB^{(2)}, \wB^{(1)}\big).
\end{equation}
This suggests that: for the $\ell_2$-RCL output \eqref{eqn:regularized-learning} to induce a small CL risk \eqref{eqn:risk}, it is necessary that the induced forgetting in \eqref{eqn:F+I} is small (or negative), and that the induced intransigence in \eqref{eqn:F+I} is also small. 
In the next section, we will show sharp bounds for both forgetting and intransigence.

\begin{remark}[Generalization to CL with $T$ tasks]
The forgetting-intransigence decomposition \eqref{eqn:F+I} is defined for a CL problem with only two tasks. 
This can be generalized easily.
Specifically, let $\wB^{(t)}$ be a sequence of model parameters generated by a CL algorithm learning the $t$-th task, $t=1,\dots,T$. Then we can define 
\begin{equation*}
    \Fcal\big( \wB^{(t)},t=1,\dots,T \big) 
    := \sum_{t=1}^{T-1} \Rcal_{t}\big(\wB^{(T)}\big) - \Rcal_{t}\big(\wB^{(t)}\big),\quad 
    \Ical\big(\wB^{(t)},t=1,\dots,T \big) 
    := \sum_{t=1}^{T} \Rcal_t\big(\wB^{(t)}\big) - \min \Rcal_t (\cdot).
\end{equation*}
Similarly, the following \emph{forgetting-intransigence decomposition} holds:
\begin{align*}
    \sum_{t=1}^T \Rcal_t\big(\wB^{(T)}\big) - \min \Rcal_t (\cdot)
    = \Fcal\big( \wB^{(t)},t=1,\dots,T \big)  +  \Ical\big(\wB^{(t)},t=1,\dots,T \big).
\end{align*}
This work focuses on $T=2$ for the conciseness of a theoretical treatment.
\end{remark}

\section{Main Results}\label{sec:main}
Before presenting the main results, we first review a set of assumptions for our analysis.

\subsection{Assumptions}

\begin{assumption}[Fixed design]\label{assump:fixed-design}
Assume that the feature vectors $\big(\xB_t^{(1)}\big)_{t=1}^n$ and $\big(\xB_t^{(2)}\big)_{t=1}^n$ are fixed and that the labels $\big(y_t^{(1)}\big)_{t=1}^n$ and $\big(y_t^{(2)}\big)_{t=1}^n$ are independent random variables. 
\end{assumption}
This work focuses on a fixed design setting as described by Assumption \ref{assump:fixed-design}, where features are fixed but the labels are still random variables \citep{hsu2012random}.
In comparison, the work by \citet{evron2022catastrophic} assumed that both features and labels are fixed.
As will be demonstrated later, even in the fixed design we already show very interesting statistic properties of CL that are unknown before our work. 
We leave the more interesting, a random design \citep{hsu2012random} analysis of CL
as a future direction.

\begin{assumption}[Shared optimal parameter]\label{assump:shared-optimal}
    Assume that there exists a $\wB^*\in\Rbb^d$ such that 
    \[\wB^*\in \arg\min \Rcal_1(\wB),\quad \wB^*\in \arg\min \Rcal_2(\wB).\]
\end{assumption}
We assume there is a common optimal parameter for the two tasks. 
This corresponds to the \emph{domain-incremental CL} \citep{van2019three}.
In comparison, \citet{evron2022catastrophic} also assumed such a condition when studying the optimization behavior of CL.
However, \citet{heckel2022provable, lin2023theory} considered a more general CL setting, where the optimal model parameters might be different for different tasks.
As shown later, the domain-incremental setting already reflects a key statistical trade-off between forgetting and intransigence in CL.

The next assumption is standard in linear regression literature, which specifies a well-specified additive label noise. 
\begin{assumption}[Well-specified noise]\label{assump:noise}
    Assume that for the $\wB^*$ in Assumption \ref{assump:shared-optimal}, it holds that: for $k=1,2$ and $t=1,\dots,n$, 
    \begin{equation*}
    \Ebb [y^{(k)}_t ]= {\xB^{(k)}_t}^\top \wB^*,\quad 
    \sigma^2 := \Ebb \big(y^{(k)}_t - {\xB^{(k)}_t}^\top \wB^* \big)^2,
    \end{equation*}
    where $\sigma^2>0$ refers to the variance of the label noise.
\end{assumption}
For conciseness, we assume that the two tasks have the same noise level. This is not restrictive, and our results can be directly generalized to allow different noise levels as well.  

\begin{assumption}[Commutable data covariance matrices]\label{assump:commutable}
Assume that \[\HB^{(1)} \HB^{(2)} = \HB^{(2)} \HB^{(1)},\ \ \text{where}\ \HB^{(1)} := \frac{1}{n}  {\XB^{(1)}}^\top \XB^{(1)} \ \text{and}\ \HB^{(2)} := {1 \over n} {\XB^{(2)}}^\top \XB^{(2)}.\] 

Denote the eigenvalues of $\HB^{(1)}$ and $\HB^{(2)}$ by $\big( \lambda_i^{(1)} \big)_{i=1}^d$ and $\big( \lambda_i^{(2)} \big)_{i=1}^d$, respectively. 
Denote the corresponding eigenvectors to be $(\vB_i)_{i=1}^d$ .
\end{assumption}
In this work, we assume that the data covariance matrices are commutable.
This is a common assumption in transfer learning literature, see, e.g., \citet{lei2021near} and Section 4 in \citet{wu2022power}.
Note that this assumption does not requires $\HB^{(1)}$ and $\HB^{(2)}$ to be diagonal.
Extensions to our results by relaxing this assumption are left as an open technical problem.

\paragraph{Notations.}
For two functions $f(x) \ge 0$ and $g(x) \ge 0$ defined for $x \in [0, \infty)$,
we write $f(x) = \Ocal (g(x))$ if $\lim_{x\to\infty} f(x) / g(x) < c$ for some absolute constant $c  >0$;
we write $f(x) = \Omega(g(x))$ if $g(x) = \Ocal (f(x))$;
and we write $f(x) = \Theta(g(x))$ if $f(x) = \Ocal (g(x))$ and $g(x) = \Ocal(f(x))$.
Moreover, we write $f(x) = o (g(x))$ if $\lim_{x\to\infty} f(x) / g(x) = 0$.

For a matrix 
\(\AB := \sum_{i=1}^d a_i \vB_i \vB_i^\top\),
where $(\vB_i)_{i\ge 1}$ are the eigenvectors of $\HB^{(1)}$ (and $\HB^{(2)}$) and $(a_i)_{i\ge 1}$ are non-negative scalars, and 
an index set $\Kbb \subset \{1,\dots,d\}$,
we define
\[
\AB_{\Kbb} := \sum_{i\in \Kbb} a_i \vB_i \vB_i^\top,\quad 
\AB_{\Kbb}^{-1} := \sum_{i\in \Kbb, a_i > 0} \frac{1}{a_i}\vB_i \vB_i^\top.
\]

\subsection{Risk Bounds for Joint Learning}
When the two datasets can be accessed simultaneously, one can solve the two-task learning problem by \emph{joint learning} (JL), which outputs $\wB_{\mathtt{joint}}$ such that 
\begin{equation}\label{eqn:joint-learning}
    \wB_{\mathtt{joint}} = \argmin_{\wB: \yB^{(1)} = \XB^{(1)} \wB, \yB^{(2)} = \XB^{(2)} \wB} \| \wB  \|_2^2.
\end{equation}
Note that JL is not a practical CL method, as it requires either simultaneous access to both datasets or equivalently, a $nd$-dimensional long-term memory vector for storing the first dataset.
As JL uses strictly more information than every CL algorithm, it can serve as an imaginary ``upper'' baseline method that describes the best possible outcome of solving a CL problem \citep{farajtabar2020orthogonal, mirzadeh2020understanding, saha2021gradient}.
We follow this convention and limit ourselves to considering only the CL problems that can be solved efficiently by JL. 
To this end, the following proposition provides a risk bound for solving two linear regression tasks with JL.

\begin{proposition}[A risk bound for JL]\label{prop:joint-learning}
Suppose that Assumptions \ref{assump:fixed-design} to \ref{assump:commutable} hold.
Then for the JL output \eqref{eqn:joint-learning}, it holds that
    \begin{equation*}
        \expect{\Rcal( \wB_\mathtt{joint} )} - \Rcal(\wB^*) = {\sigma^2 \over n} \cdot \mathrm{rank}(\HB^{(1)} + \HB^{(2)}).
    \end{equation*}
\end{proposition}

\paragraph{The effect of JL.}
According to Proposition \ref{prop:joint-learning}, for JL to achieve an $o(1)$ excess risk, it is necessary (and also sufficient) to have that 
\begin{equation}\label{eqn:JL-condition}
    \mathrm{rank}(\HB^{(1)}) = o(n),\quad \mathrm{rank}(\HB^{(2)}) = o(n).
\end{equation}
As discussed earlier, for a CL problem to be learnable by a valid CL algorithm, it has to be learnable by the JL algorithm. 
Therefore, in what follows when we analyze CL algorithms, we will only focus on the CL problems where \eqref{eqn:JL-condition} holds.

\subsection{Risk Bounds for Continual Learning} \label{sec:risk-CL}

In this part, we present our main results.
We first provide general risk bounds (Theorem \ref{thm:regularized-learning})  for $\ell_2$-RCL \eqref{eqn:regularized-learning} with any regularization parameter $\mu \ge 0$; in particular, they allow $\mu=0$, therefore, cover OCL \eqref{eqn:sequential-learning} as a special case.

\begin{theorem}[A risk bound for $\ell_2$-RCL/OCL]\label{thm:regularized-learning}
    Suppose that Assumptions \ref{assump:fixed-design} to \ref{assump:commutable} hold.
    Then for the $\ell_2$-RCL output \eqref{eqn:regularized-learning}, it holds that 
    \begin{equation*}
        \Ebb\big[ \Rcal\big( \wB^{(2)} \big) \big]  - \min \Rcal (\cdot) = \Ebb \big[ \Fcal\big(\wB^{(2)}, \wB^{(1)}\big) \big] + \Ebb \big[ \Ical\big(\wB^{(2)}, \wB^{(1)}\big) \big] .
    \end{equation*}
    Moreover, it holds that
    \begin{align*}
        \Ebb \big[ \Fcal\big(\wB^{(2)}, \wB^{(1)}\big) \big]
        &=  \frac{\sigma^2}{n} \cdot \dimF, \\
        \Ebb \big[ \Ical\big(\wB^{(2)}, \wB^{(1)}\big) \big] 
        &= \bias + \frac{\sigma^2}{n} \cdot \dimI,
    \end{align*}
    where
    \begin{equation}\label{eqn:eff-dim}
    \begin{aligned}
        \dimF & :=   \tr\Big\{\IB_{\Kbb}\cdot \big( \HB^{(1)}-\HB^{(2)}-2\mu\IB\big) \cdot \HB^{(2)}\cdot \big( \HB^{(2)} + \mu \IB \big)^{-2} \Big\}, \\
        \dimI & := \tr\Big\{ \big( \mu^2 \cdot {\HB^{(1)}}^{-1} + \HB^{(2)}\big)\cdot \HB^{(2)}\cdot \big( \HB^{(2)} + \mu\IB \big)^{-2} \Big\} + \rank\big\{ \HB^{(1)} \big\}, \\
        \bias & := (\wB^*)^\top \cdot \IB_{\Kbb^c} \cdot \mu^2 \cdot \HB^{(2)} \cdot \big(  {\HB^{(2)} +  \mu\IB}\big)^{-2} \cdot \wB^*,
    \end{aligned}
    \end{equation}
    for index sets $\Kbb := \big\{ i: \lambda_i^{(1)} >  0 \big\}$.
    
    In particular, when the regularization parameter $\mu = 0$, i.e., for the OCL output \eqref{eqn:sequential-learning}, it holds that
    \begin{equation}\label{eqn:eff-dim-ocl}
    \dimF =  \tr\big\{ \HB^{(1)}{\HB^{(2)}}^{-1} \big\} - \rank \big\{\HB^{(1)}\HB^{(2)} \big\},\quad 
    \dimI = \rank\big\{ \HB^{(2)} \big\} + \rank\big\{ \HB^{(1)} \big\},\quad 
    \bias = 0.
    \end{equation}
\end{theorem}

Theorems \ref{thm:regularized-learning} tightly characterizes the performance of $\ell_2$-RCL and OCL methods in the fixed design setting.
Specifically, the CL risk is decomposed into a forgetting part and an intransigence part, and sharp bounds on both parts are established. 
In the following, we discuss the effects of OCL and $\ell_2$-RCL in turn. 

\paragraph{The effect of OCL.}

Consider \eqref{eqn:eff-dim-ocl} in Theorem \ref{thm:regularized-learning} that characterizes the forgetting and intransigence induced by OCL.
When the CL problem can be solved by JL, i.e., \eqref{eqn:JL-condition} holds, it is clear that $\dimI$ is small, i.e., OCL has a small intransigence. 
Then the performance of OCL is dominated by its forgetting according to Theorem \ref{thm:regularized-learning}. 
In particular, for OCL to achieve an $o(1)$ CL excess risk, it is sufficient (and also necessary, as $\rank \big\{\HB^{(1)}\HB^{(2)} \big\} = o(n)$ under \eqref{eqn:JL-condition}) to have 
\begin{equation}\label{eq:ocl:task-similarity}
     \dimF \le \tr\big\{ \HB^{(1)} {\HB^{(2)}}^{-1} \big\} = o(n). 
\end{equation}
This requires the small eigenvalue space of $\HB^{(2)}$ to be low-dimensional and/or to be aligned with the small eigenvalue space of $\HB^{(1)}$.
On the other hand, if \eqref{eq:ocl:task-similarity} does not hold, then OCL suffers from constant forgetting, i.e., catastrophic forgetting happens. 
From this perspective, \eqref{eq:ocl:task-similarity} provides a precise task similarity measurement, such that OCL is able to solve CL problems with similar tasks but is unable to solve CL problems with dissimilar tasks due to catastrophic forgetting.

For understanding the performance of $\ell_2$-RCL, the following corollary is useful.
\begin{corollary}[Risk upper bounds for $\ell_2$-RCL]\label{thm:rcl:upper-bound}
    Under the same setup of Theorem \ref{thm:regularized-learning}, it holds that
    \begin{equation}\label{eqn:eff-dim-upper-bound}
    \begin{aligned}
        \dimF
        &\le \tr \Big\{ \HB^{(1)} \cdot \HB^{(2)}\cdot \big( \HB^{(2)} + \mu \IB \big)^{-2}\Big\} \\
        &\le  \tr \Big\{ \HB^{(1)} \cdot  \big( \HB^{(2)} + \mu \IB \big)^{-1}\Big\},\\
        \dimI 
        &\le \tr\Big\{ \mu^2 \cdot {\HB^{(1)}}^{-1}\cdot \HB^{(2)}\cdot \big( \HB^{(2)} + \mu\IB \big)^{-2} \Big\} + \rank\big\{ \HB^{(2)} \big\} + \rank\big\{ \HB^{(1)}\big\} \\
        &\le \tr\Big\{ \mu \cdot {\HB^{(1)}}^{-1} \Big\} + \rank\big\{ \HB^{(2)} \big\} + \rank\big\{ \HB^{(1)}\big\},\\
        \bias 
        &\le \mu \cdot \| \wB^*\|_2^2.
    \end{aligned}
    \end{equation}
\end{corollary}

\paragraph{The effect of $\ell_2$-RCL.}
We now discuss the effect of $\ell_2$-RCL based on the forgetting and intransigence bounds in Theorem \ref{thm:regularized-learning} and their  upper bounds in Corollary \ref{thm:rcl:upper-bound}.
Firstly, according to \eqref{eqn:eff-dim} (or \eqref{eqn:eff-dim-upper-bound}), the regularization parameter $\mu$ adjusts the trade-off between the forgetting and intransigence parts in the risk bound of $\ell_2$-RCL.
Specifically, a large $\mu$ reduces $\dimF$ but promotes $\dimI$ and $\bias$, i.e., the forgetting becomes smaller but the intransigence becomes larger, and vice versa.
Our theory on the effect of the regularization parameter in $\ell_2$-RCL is consistent with the empirical observation for the regularization-based CL methods \citep{kirkpatrick2017overcoming,kolouri2020sliced,li2021lifelong}.

Secondly, we compare the effect of $\ell_2$-RCL with that of OCL.
According to Theorem \ref{thm:regularized-learning}, for $\ell_2$-RCL to achieve an $o(1)$ CL excess risk, it suffices to choose a regularization parameter $\mu$ such that $\dimF, \dimI \le o(n)$ and that $\bias \le o(1)$. 
Consider the upper bounds for these terms in \eqref{eqn:eff-dim-upper-bound}, and under the additional assumption that the problem can be solved by JL (i.e., \eqref{eqn:JL-condition} holds), 
it suffices to require that
\begin{equation}\label{eq:rcl:task-similarity}
     \tr \Big\{ \HB^{(1)} \cdot  \big( \HB^{(2)} + \mu \IB \big)^{-1}\Big\} \le o(n), \quad 
     \mu \cdot \tr\Big\{ {\HB^{(1)}}^{-1} \Big\} \le o(n),\quad 
    \mu \le o(1),
\end{equation}
holds for some $\mu\ge 0$.
Clearly, \eqref{eq:rcl:task-similarity} is a weaker set of requirements than \eqref{eq:ocl:task-similarity} as the latter corresponds to the former with $\mu=0$.
Moreover, even when the two tasks in CL are dissimilar such that OCL cannot achieve a small risk, i.e., \eqref{eq:ocl:task-similarity} does not hold, 
$\ell_2$-RCL is still possible to achieve a small risk as long as \eqref{eq:rcl:task-similarity} holds for some $\mu>0$.
For example, as long as $\tr\{\HB^{(1)}\} = o(n)$ and $\tr\{\HB^{(1)}\}\cdot \tr\{{\HB^{(1)}}^{-1}\} = o(n^2)$, \eqref{eq:rcl:task-similarity} can be made true for a well tuned $\mu$.
Note that in this example, $\HB^{(2)}$ might not be well aligned with $\HB^{(1)}$ and \eqref{eq:ocl:task-similarity} might not hold.
These suggest that a well-tuned $\ell_2$-RCL achieves a trade-off between forgetting and intransigence superior to that achieved by OCL.

Thirdly, it is worth noting that there exist CL problems that can be solved by JL but not by $\ell_2$-RCL, even with an optimally tuned regularization parameter.  
See, e.g., Corollary \ref{cor:regularized-limitation}. This demonstrates the limitation of $\ell_2$-RCL.

\begin{figure}[t]
    \centering
    \subfigure[$\mathtt{Q}(5,0)$]{\includegraphics[width=0.32\linewidth]{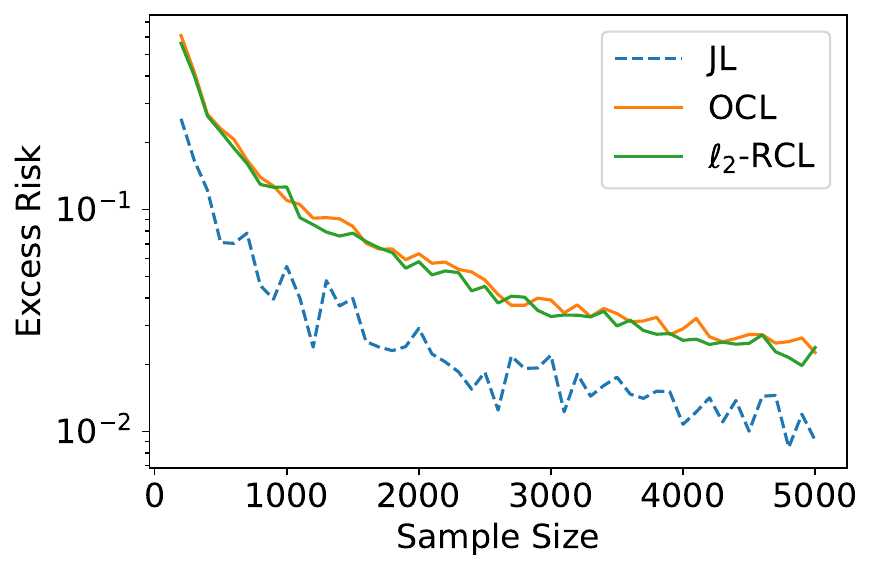}\label{fig:easy-change-n}}
    \subfigure[$\mathtt{Q}(15,0)$]{\includegraphics[width=0.32\linewidth]{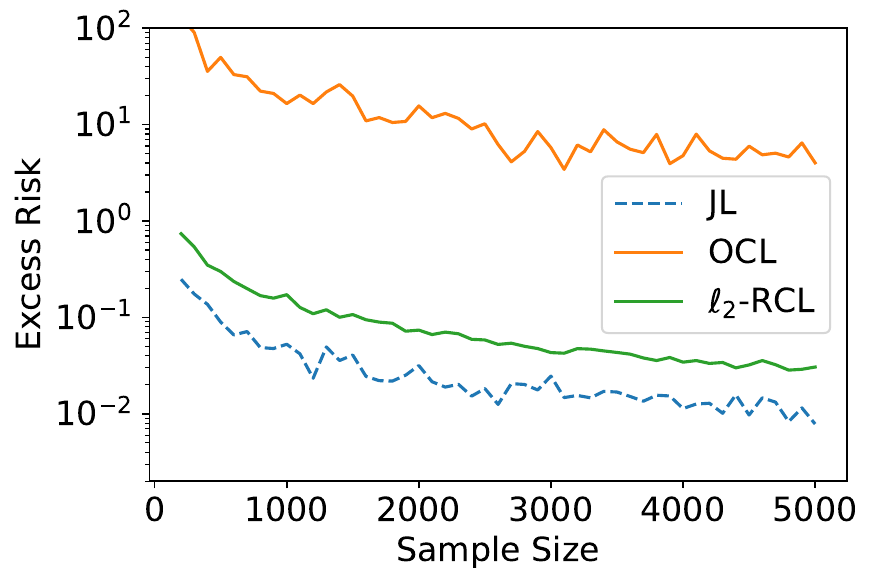}\label{fig:medium-change-n}}
    \subfigure[$\mathtt{Q}(15,15)$]{\includegraphics[width=0.32\linewidth]{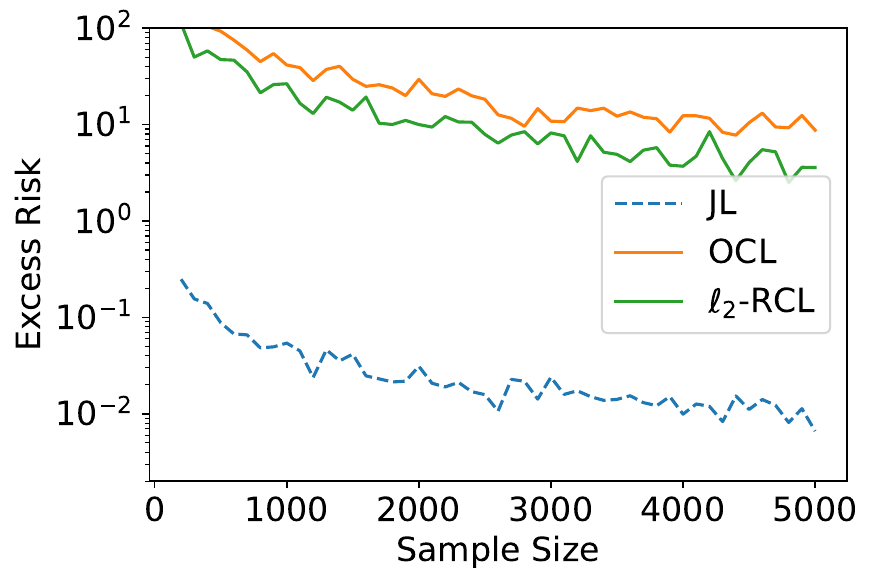}\label{fig:hard-change-n}}
    \caption{
    Expected excess risk vs.\ sample size for JL, OCL and $\ell_2$-RCL.
    For each point in each curve, the x-axis represents the sample size $n$ and the y-axis represents the expected CL excess risk (in the logarithmic scale).
    The problem instances $\mathtt{Q}(5,0)$, $\mathtt{Q}(15,0)$ and $\mathtt{Q}(15,15)$ are defined by \eqref{eqn:setting-q}.
    The regularization parameter in the $\ell_2$-RCL algorithm is optimally tuned. 
    The dimension of the task is $d=200$.
    The expectation of the excess risk is computed by taking an empirical average over $20$ independent runs.
    }
    \label{fig:change-n}
\end{figure}

\subsection{Examples and Numerical Simulations} 

In this part, we give specific examples to further validate the above discussions on the trade-off between forgetting and intransigence induced by OCL and $\ell_2$-RCL.
We show that catastrophic forgetting happens when the two tasks are dissimilar. 

\paragraph{When is forgetting catastrophic for OCL?}

We have shown that \eqref{eq:ocl:task-similarity} is a sufficient and necessary condition that determines when forgetting is catastrophic for OCL. 
We further demonstrate this with the following examples.

\begin{example}\label{cor:sequential-example}
    Under the same setting as Theorem \ref{thm:regularized-learning}, additionally, assume that 
    \[\sigma^2 = \Theta(1), \quad  
    \|\wB^*\|_2 = \Ocal(1), \quad  
    \rank \big\{ \HB^{(1)} \big\} =\Ocal(1),\quad 
    \rank \big\{ \HB^{(2)}\big\} =\Ocal(1). \] 
    Consider the OCL output \eqref{eqn:sequential-learning} (or \eqref{eqn:regularized-learning} with $\mu=0$), then:
    \begin{enumerate}[leftmargin=*]
        \item If $\|\HB^{(1)}\|_2 = \Ocal(1)$ and all non-zero eigenvalues of $\HB^{(2)}$ are $\Omega(n^{-1+r})$ for some constant $0<r\leq 1$, then $\Ebb[\Rcal(\wB^{(2)})] - \Rcal (\wB^*) = \bigO{n^{-r}}$. Moreover, it holds that 
        \[ \Ebb [ \Fcal(\wB^{(2)},\wB^{(1)})]  = \Ocal(n^{-r}),\quad  \Ebb [ \Ical(\wB^{(2)},\wB^{(1)})] = \Ocal(n^\inv).\]  
        \item If $\lambda_1^{(1)} = 1$ and $\lambda_1^{(2)} = n^\inv$, then $\Ebb[\Rcal(\wB^{(2)})] - \Rcal (\wB^*) = \bigOmega{1}$. Moreover, it holds that
        \[\Ebb [ \Fcal(\wB^{(2)},\wB^{(1)})]  = \Omega(1),\quad \Ebb [ \Ical(\wB^{(2)},\wB^{(1)})]  = \Ocal( n^\inv). \] 
    \end{enumerate}
\end{example}
In the first problem in Example \ref{cor:sequential-example}, $\HB^{(2)}$ is well conditioned, so it is ``similar'' to $\HB^{(1)}$ according to \eqref{eq:ocl:task-similarity}, and the forgetting is small for OCL. 
In the second problem in Example \ref{cor:sequential-example}, there exists a direction where $\HB^{(2)}$ and $\HB^{(1)}$ are highly dissimilar, then the forgetting becomes catastrophic for OCL.

\paragraph{When can $\ell_2$-RCL temper catastrophic forgetting?}
We now present an example where the forgetting is catastrophic for OCL but the issue can be mitigated by $\ell_2$-RCL via introducing intransigence.

\begin{example}\label{cor:regularized-discussion}
    Under the same setting as Theorem \ref{thm:regularized-learning}, additionally, assume that
    \begin{gather*}
        \wB^* = (0,1,0,0,\dots), \quad \sigma^2 = 1, \\
        \HB^{(1)} = \diag(1,0,\underbrace{1,\dots,1}_{\Theta(1) \textrm{ copies}},0,0,\dots), \quad \HB^{(2)} = \diag(n^\inv, n^{-2/3}, \underbrace{1,\dots,1}_{\Theta(1) \textrm{ copies}},0,0,\dots).
    \end{gather*}
    Then: 
    \begin{itemize}
        \item For the OCL output \eqref{eqn:sequential-learning}, it holds that 
        $\Ebb[\Rcal(\wB^{(2)})] - \Rcal (\wB^*) = \bigOmega{1}$; moreover,
        \[ \Ebb [ \Fcal(\wB^{(2)},\wB^{(1)})]  = \Omega(1),\quad  \Ebb [ \Ical(\wB^{(2)},\wB^{(1)})] = \Ocal(n^\inv).\]  
        \item For the $\ell_2$-RCL output \eqref{eqn:regularized-learning} with $\mu=n^{-2/3}$, it holds that  $\Ebb[\Rcal(\wB^{(2)})] - \Rcal (\wB^*) = \bigO{n^{-2/3}}$; moreover,
        \[ \Ebb [ \Fcal(\wB^{(2)},\wB^{(1)})]  = \Ocal(n^{-2/3}),\quad  \Ebb [ \Ical(\wB^{(2)},\wB^{(1)})] = \Ocal(n^{-2/3}).\]
    \end{itemize}
\end{example}
In Example \ref{cor:regularized-discussion}, we see that 
OCL suffers from a constant CL excess risk due to catastrophic forgetting, while $\ell_2$-RCL achieves an $o(1)$ CL excess risk by balancing forgetting and intransigence. 
This example illustrates that $\ell_2$-RCL is able to (partially) mitigate the issue of catastrophic forgetting by introducing intransigence.

\paragraph{When is $\ell_2$-RCL not able to temper catastrophic forgetting?}
We next present a CL problem where the forgetting and intransigence cannot be small simultaneously for $\ell_2$-RCL, even with an optimally tuned regularization parameter $\mu$.

\begin{example}\label{cor:regularized-limitation}
    Under the same setting as Theorem \ref{thm:regularized-learning}, additionally, assume that 
    \begin{gather*}
        \|\wB^*\|_2 \leq 1, \quad \sigma^2 = 1, \quad 
        \HB^{(1)} = \diag(1,n^{-2},0,0,\dots), \quad \HB^{(2)} = \diag(n^\inv, n^{-1}, 0, 0,\dots).
    \end{gather*}
    Then for the $\ell_2$-RCL output \eqref{eqn:regularized-learning} with any $\mu>0$, it holds that $\Ebb[\Rcal(\wB^{(2)})] - \Rcal (\wB^*) =\Omega(1)$.
    In special cases:
    \begin{itemize}
        \item if $\mu = \bigO{n^{-1.5}}$, then $\Ebb [ \Fcal(\wB^{(2)},\wB^{(1)})]  = \bigOmega{1}$ and $ \Ebb [ \Ical(\wB^{(2)},\wB^{(1)})]  = \bigO{n^\inv}$;
        \item if $\mu = \bigTheta{n^{-1}}$, then $\Ebb [ \Fcal(\wB^{(2)},\wB^{(1)})]  = \bigOmega{1}$ and $ \Ebb [ \Ical(\wB^{(2)},\wB^{(1)})]  = \bigOmega{1}$;
        \item if $\mu = \bigTheta{n^{-0.5}}$, then $\Ebb [ \Fcal(\wB^{(2)},\wB^{(1)})]  = \bigO{n^\inv}$ and $ \Ebb [ \Ical(\wB^{(2)},\wB^{(1)})]  = \bigOmega{1}$.
    \end{itemize}
\end{example}
Example \ref{cor:regularized-limitation} demonstrates the limitation of $\ell_2$-RCL: for $\ell_2$-RCL on some CL problems, 
while it is possible to choose a small $\mu$ such that the intransigence is small or a large $\mu$ such that the forgetting is small, it is not possible to make intransigence and forgetting small simultaneously. 
We conjecture that such a hard CL problem can be solved with more comprehensive regularization methods that utilize the Hessian information \citep{kirkpatrick2017overcoming,kolouri2020sliced,li2021lifelong}. 
The limitation of $\ell_2$-RCL has also been observed in practical CL problems \citep{hsu2018re}.

\paragraph{Numerical simulations.}
We design a set of experiments to verify our theory.
We construct a set of CL problems where the data covariance matrices $\HB^{(1)}$ and $\HB^{(2)}$ have different levels of similarity. 
The design is motivated by \citet{wu2022power}.
Specifically, we fix $d=200$. 
The eigenvalues for $\HB^{(1)}$ and $\HB^{(2)}$ are given by  $\{ 1/2^{i-1}\}_{i=1}^d$, and we control the similarity between $\HB^{(1)}$ and $\HB^{(2)}$ by the alignment between their eigenvalues.
Mathematically, for two integers $k,l>0$, 
we specify a CL problem, $\mathtt{Q}(k,l)=(\wB^*, \sigma^2, \HB^{(1)}, \HB^{(2)})$, by:
\begin{equation}\label{eqn:setting-q}
\begin{gathered}
    \HB^{(1)} = \diag \bigg( 1,{1\over 2},{1\over 2^2},\dots, {1\over 2^{k-1}},{1\over 2^k},\dots, {1\over 2^{k+l-1}}, {1\over 2^{k+l}}, \dots \bigg), \\
    \HB^{(2)} = \diag \bigg( {1\over 2^{k-1}}, {1\over 2^{k-2}},\dots, 1, {1\over 2^{k+l-1}}, ,\dots, {1\over 2^k}, {1\over 2^{k+l}}, \dots \bigg), \\
    \wB^* = \bigg( \underbrace{1,\dots,1}_{k+l+2 \textrm{ copies}},\frac{1}{k+l+3},\frac{1}{k+l+4},\dots \bigg)^\top,
    \qquad \sigma^2 = 1.
\end{gathered}
\end{equation}
It is clear that ranging from $\mathtt{Q}(5,0)$ to $\mathtt{Q}(15,0)$ to $\mathtt{Q}(15,15)$, $\HB^{(1)}$ and $\HB^{(2)}$ become more and more dissimilar. 
We then apply JL, OCL, and $\ell_2$-RCL to the three problem instances, and measure their achieved CL excess risks versus the number of samples $n$. 
The results are shown in Figure \ref{fig:change-n}.
We make the following observations from Figure \ref{fig:change-n}:
\begin{itemize}[leftmargin=*]
    \item For the problem $\mathtt{Q}(5,0)$ where the two tasks are very similar, all three methods achieve a vanishing CL excess risk. Moreover, the risk convergence rate of OCL and $\ell_2$-RCL matches that of JL, ignoring constant factors.  
    This suggests that CL with similar tasks is easy and can be solved by even OCL.  
    \item For problem $\mathtt{Q}(15, 0)$ where the two tasks are not similar, OCL fails to compete with JL due to catastrophic forgetting. 
    However, $\ell_2$-RCL can still match the convergence rate of JL by balancing the forgetting and intransigence. 
    \item For problem $\mathtt{Q}(15, 15)$ where the two tasks are very dissimilar, there is a sharp gap in terms of the convergence rate between JL and OCL/$\ell_2$-RCL.
    This demonstrates the limitation of $\ell_2$-RCL for solving very hard CL problems. 
\end{itemize}

\begin{figure}[t]
    \centering
    \subfigure[$\mathtt{Q}(15,0)$]{\includegraphics[width=0.32\linewidth]{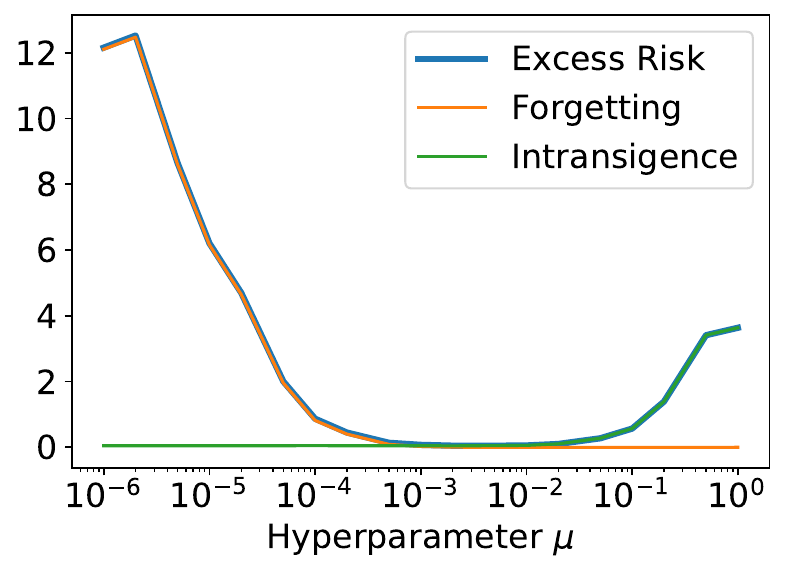}\label{fig:medium-change-mu}}
    \subfigure[$\mathtt{Q}(15,15)$]{\includegraphics[width=0.32\linewidth]{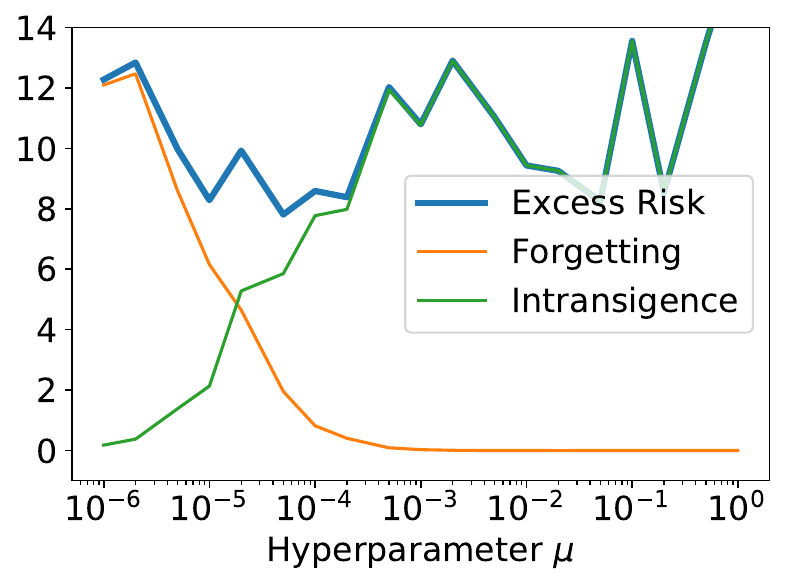}\label{fig:hard-change-mu}}
    \caption{
    The trade-off between forgetting and intransigence for $\ell_2$-RCL.
    For each point in each curve, the x-axis represents the sample size $n$ and the y-axis represents the expected CL excess risk or the expected forgetting or the expected intransigence.
    The problem instances $\mathtt{Q}(15,0)$ and $\mathtt{Q}(15,15)$ are defined by \eqref{eqn:setting-q}.
    The dimension of the task is $d=200$.
    The sample size is $n=2,000$.
    The expectation is computed by taking an empirical average over $20$ independent runs.
    }
    \label{fig:change-mu}
\end{figure}

In order to understand the different behaviors of $\ell_2$-RCL in problem instances $\mathtt{Q}(15,0)$ and $\mathtt{Q}(15,15)$, 
we further compute the (expected) forgetting and intransigence of $\ell_2$-RCL in these two problem instances with a varying regularization parameter ($\mu$). 
The results are shown in Figure \ref{fig:change-mu}.
From 
Figure \ref{fig:change-mu}, we see that for both problem instances, $\ell_2$-RCL is able to reduce forgetting by using a large $\mu$, or reduce intransigence by using a small $\mu$.
However, for $\mathtt{Q}(15,0)$, there is a $\mu$ that simultaneously reduces forgetting and intransigence, and for $\mathtt{Q}(15,15)$, forgetting and intransigence cannot be reduced simultaneously.
This explains why $\ell_2$-RCL can solve $\mathtt{Q}(15,0)$ but cannot solve $\mathtt{Q}(15,15)$.

\section{Related Works}

\paragraph{Continual learning in practice.} 
Our theoretical framework for CL is motivated by the algorithmic research for CL (see, e.g., \citet{parisi2019continual}). 
Among them, the most relevant to us is the \emph{regularization based} CL method \citep{kirkpatrick2017overcoming, zenke2017continual,aljundi2018memory, liu2018rotate, ritter2018online,kolouri2020sliced,li2021lifelong},
where a regularization term is adopted to prevent the model from deviating from the old one when learning a new task, at the cost of introducing intransigence. 
Along this line, the research focus is to find the best regularization for achieving a good balance between forgetting and intransigence.
For example, the \emph{Elastic Weight Consolidation} (EWC) algorithm \citep{kirkpatrick2017overcoming} uses a weighted $\ell_2$-regularization, where the weight corresponds to the Hessian diagonal from the last model, and the \emph{Sliced Cramer Preservation} (SCP) algorithm \citep{kolouri2020sliced} considers a regularization defined by the sliced Cramer distance instead of an $\ell_2$-distance.
Despite the advances they have made in practice, a statistical continual learning theory is still lacking. 
Our work fills this gap by establishing a statistical framework for studying the performance of regularization-based CL algorithms (e.g., $\ell_2$-RCL). 

In addition to the regularization-based CL method, the memory based \citep{rebuffi2017icarl,chaudhry2019tiny,shin2017continual,chaudhry2018efficient,farajtabar2020orthogonal,saha2021gradient} and the architecture based \citep{rusu2016progressive,xu2018reinforced,mallya2018packnet, serra2018overcoming, yoon2020scalable} CL methods are also popular in practice. 
These two categories of CL methods are beyond the focus in this paper, but we believe our problem formulation can be migrated to analyze them as well.

\paragraph{Continual learning in theory.}
Compared to the development of CL in practice, the theoretical results for CL are less rich \citep{bennani2020generalisation,doan2021theoretical,lee2021continual,chen2022memory,evron2022catastrophic,heckel2022provable,yin2020optimization,lin2023theory}.
We now discuss their relationship to our work in turn.

The work by \citet{evron2022catastrophic} is most relevant to us. 
We both consider a domain-incremental CL problem \citep{van2019three} with linear regression tasks. 
However, there are two notable differences. 
Firstly, the label is assumed to be a random variable in our setting (see Assumption \ref{assump:fixed-design}), while the entire dataset is considered to be fixed in their setting.
So they only studied the optimization aspects of CL, while we focus on the statistical aspects of CL. 
Secondly, \citet{evron2022catastrophic} only considered the OCL method \eqref{eqn:sequential-learning}, but we study a more interesting RCL algorithm \eqref{eqn:regularized-learning} that reveals a richer theoretical picture of CL, e.g., the regularization parameter can adjust the trade-off between forgetting intransigence (see Theorem \ref{thm:regularized-learning}).

The works by \citet{heckel2022provable,lin2023theory} considered a task-incremental CL setting \citep{van2019three}, where the true model parameters are different across tasks; however, they additionally assumed that the data covariance matrices are the same (an identity matrix) for all tasks. 
In comparison, we focus on a domain-incremental CL setting \citep{van2019three}, where the true model parameters are shared for all tasks but the data covariance matrices are different across tasks. 
So our results are not directly comparable to those in \citet{heckel2022provable,lin2023theory}. 
In particular, we emphasize that the forgetting-intransigence trade-off shown in our work has not appeared in their works.

The work by \citet{yin2020optimization} considered CL with a Hessian-aware regularization. 
However, their main focus is still the optimization behavior. Their statistical results are obtained by standard uniform convergence analysis, which does not reveal the trade-off between forgetting and intransigence as shown in our work.

The remaining papers are not directly related to ours. 
Specifically, the works by \citet{bennani2020generalisation,doan2021theoretical} analyzed the generalization error and the forgetting of the \emph{Orthogonal Gradient Descent} (OGD) method \citep{farajtabar2020orthogonal}.
\citet{lee2021continual} studied CL in a multi-head setting.
Finally, \citet{chen2022memory} showed a negative, memory lower bound for CL.

\section{Conclusion}
We consider an $\ell_2$-regularized continual learning with two linear regression tasks in the fixed design setting.
We derive sharp risk bounds for the algorithm. 
We show a provable trade-off between induced forgetting and intransigence, which can be adjusted by the regularization parameter. 
We show that forgetting could be catastrophic when the tasks are dissimilar and that an $\ell_2$-regularization (partially) mitigates the issue by introducing intransigence. 
We demonstrate the power and limitation of $\ell_2$-regularized continual learning with concrete examples.

\subsubsection*{Acknowledgments}
This research has been partially supported by the Defense Advanced Research Projects Agency (DARPA) under Contract No HR00112190130, Office of Naval Research under award N00014-23-1-2737, and the National Science Foundation Grant under award 2244899.

\bibliography{ref}
\bibliographystyle{collas2023_conference}

\newpage
\appendix

\section{Missing Proofs}

\subsection{Proof of Theorem \ref{thm:regularized-learning}}
\begin{proof}[Proof of Theorem \ref{thm:regularized-learning}]~
The proof can be done by direct computation.

\paragraph{Computing $\wB^{(1)}$.}
We first compute $\wB^{(1)}$. By definition, we have 
\begin{align*}
    \wB^{(1)} &= \big((\XB^{(1)})^\top\XB^{(1)}  \big)^{-1} (\XB^{(1)})^\top \yB \\
    &= \big((\XB^{(1)})^\top\XB^{(1)}  \big)^{-1} (\XB^{(1)})^\top \XB^{(1)}\wB^* + \big((\XB^{(1)})^\top\XB^{(1)}  \big)^{-1} (\XB^{(1)})^\top\epsilonB^{(1)} \\
    &= \IB_{\Kbb} \cdot \wB^* + \frac{1}{n} { \HB^{(1)} }^{-1} {\XB^{(1)}}^\top \epsilonB^{(1)}.
\end{align*}
Therefore
\begin{align}
     \wB^{(1)} - \wB^* 
     = - \IB_{\Kbb^c} \cdot \wB^* + \frac{1}{n} { \HB^{(1)} }^{-1} {\XB^{(1)}}^\top \epsilonB^{(1)}.\notag 
\end{align}
Now noticing that 
\[\Ebb\epsilonB^{(1)} = 0,\quad 
\Ebb \epsilonB^{(1)}{\epsilonB^{(1)}}^\top = \sigma^2 \IB,
\]
so the covariance of $\wB^{(1)} - \wB^* $ is
\begin{align}
    \Ebb (\wB^{(1)} - \wB^*)(\wB^{(1)} - \wB^* )^\top 
    &= \IB_{\Kbb^c} \cdot \wB^*{\wB^*}^\top \cdot \IB_{\Kbb^c} + \frac{\sigma^2 }{n^2} { \HB^{(1)} }^{-1} {\XB^{(1)}}^\top \XB^{(1)}{ \HB^{(1)} }^{-1} \notag \\
    &= \IB_{\Kbb^c} \cdot \wB^*{\wB^*}^\top \cdot \IB_{\Kbb^c} + \frac{\sigma^2}{n}\cdot { \HB^{(1)} }^{-1}. \label{eq:w1-w*:cov}
\end{align}

\paragraph{Computing $\wB^{(2)}$.}
Then we compute $\wB^{(2)}$, which is the solution of \eqref{eqn:regularized-learning}. By the first-order optimality condition, we have 
\begin{align*}
    \frac{1}{n} {\XB^{(2)}}^\top \big( \XB^{(2)} \wB^{(2)} - \yB^{(2)} \big) + \mu (\wB^{(2)} - \wB^{(1)}) = 0,
\end{align*}
which implies 
\begin{align*}
    \wB^{(2)} 
    &= \bigg( \frac{1}{n}{\XB^{(2)}}^\top \XB^{(2)}  + \mu \IB \bigg)^{-1} \bigg( \frac{1}{n}{\XB^{(2)}}^\top \yB^{(2)} + \mu \wB^{(1)} \bigg)\\
    &= \bigg( \frac{1}{n}{\XB^{(2)}}^\top \XB^{(2)}  + \mu \IB \bigg)^{-1} \bigg( \frac{1}{n}{\XB^{(2)}}^\top\XB^{(2)} \wB^* + \frac{1}{n}{\XB^{(2)}}^\top\epsilonB^{(2)}  + \mu \wB^{(1)} \bigg) \\
    &= \big( \HB^{(2)} + \mu \IB \big)^{-1}\bigg(  \HB^{(2)}\wB^* + \frac{1}{n}{\XB^{(2)}}^\top\epsilonB^{(2)}  + \mu \wB^{(1)}  \bigg).
\end{align*}
Then we have 
\begin{align*}
    \wB^{(2)}  - \wB^*
    &= \big( \HB^{(2)} + \mu \IB \big)^{-1}\bigg(  \HB^{(2)}\wB^* + \frac{1}{n}{\XB^{(2)}}^\top\epsilonB^{(2)}  + \mu \wB^{(1)} -  \big( \HB^{(2)} + \mu \IB \big)\wB^*\bigg) \\
    &= \big( \HB^{(2)} + \mu \IB \big)^{-1}\bigg(  \mu (\wB^{(1)}-\wB^*)  + \frac{1}{n}{\XB^{(2)}}^\top\epsilonB^{(2)} \bigg).
\end{align*}
Similarly, noticing that 
\[\Ebb\epsilonB^{(2)} = 0,\quad 
\Ebb \epsilonB^{(2)}{\epsilonB^{(2)}}^\top = \sigma^2 \IB,
\]
so the covariance of $\wB^{(2)} - \wB^* $ is
\begin{align}
    \Ebb (\wB^{(2)}  - \wB^*)(\wB^{(2)}  - \wB^*)^\top
    &= \mu^2\cdot \big( \HB^{(2)} + \mu \IB \big)^{-1} \cdot \Ebb (\wB^{(1)}  - \wB^*)(\wB^{(1)}  - \wB^*)^\top \cdot \big( \HB^{(2)} + \mu \IB \big)^{-1} \notag \\
    &\qquad + \frac{\sigma^2}{n^2} \big( \HB^{(2)} + \mu \IB \big)^{-1} \cdot {\XB^{(2)}}^\top \XB^{(2)} \cdot \big( \HB^{(2)} + \mu \IB \big)^{-1}  \notag \\
    &= \mu^2\cdot \big( \HB^{(2)} + \mu \IB \big)^{-1} \cdot \Ebb (\wB^{(1)}  - \wB^*)(\wB^{(1)}  - \wB^*)^\top \cdot \big( \HB^{(2)} + \mu \IB \big)^{-1} \notag \\
    &\qquad + \frac{\sigma^2}{n} \big( \HB^{(2)} + \mu \IB \big)^{-2}  \HB^{(2)}.\label{eq:w2-w*:cov}
\end{align}

\paragraph{Risk decomposition.}
According to the risk definition and the assumptions on the noise, we have 
\begin{align}
    \Rcal_1(\wB) &= \frac{1}{n}\Ebb\big\| \XB^{(1)}\wB - \yB^{(1)} \big\|_2^2 \notag \\  
    &= \frac{1}{n}\Ebb\big\| \XB^{(1)}\wB - \XB^{(1)}\wB^* - \epsilonB^{(1)} \big\|_2^2 \notag \\ 
    &= (\wB - \wB^*)^\top \HB^{(1)} (\wB - \wB^*) + \sigma^2 \notag \\
    &= \la\HB^{(1)},\ (\wB - \wB^*)(\wB - \wB^*)^\top \ra + \sigma^2. \label{eq:risk1}
\end{align}
Similarly, the risk for the second task is 
\begin{align}
    \Rcal_2 (\wB) =  \la\HB^{(2)},\ (\wB - \wB^*)(\wB - \wB^*)^\top \ra + \sigma^2. \label{eq:risk2}
\end{align}

Based on \eqref{eq:risk1} and \eqref{eq:risk2}, we can compute the forgetting and intransigence as follows: 
\begin{align}
    \Fcal &:= \Ebb [ \Rcal_1 (\wB^{(2)}) - \Rcal_1 (\wB^{(1)}) ] \notag \\ 
    &= \la \HB^{(1)},\ \Ebb (\wB^{(2)} - \wB^*)(\wB^{(2)} - \wB^*)^\top \ra -  \la \HB^{(1)},\ \Ebb (\wB^{(1)} - \wB^*)(\wB^{(1)} - \wB^*)^\top \ra \label{eq:forgetting},
\end{align}
and 
\begin{align}
    \Ical &:= \Ebb [ \Rcal_2 (\wB^{(2)})] - \min \Rcal_2 + \Ebb [\Rcal_1 (\wB^{(1)}) ] - \min \Rcal_1 \notag  \\ 
    &= \la \HB^{(2)},\ \Ebb (\wB^{(2)} - \wB^*)(\wB^{(2)} - \wB^*)^\top \ra +  \la \HB^{(1)},\ \Ebb (\wB^{(1)} - \wB^*)(\wB^{(1)} - \wB^*)^\top \ra \label{eq:instransigence}.
\end{align}

\paragraph{Computing forgetting.}
We now compute forgetting according to \eqref{eq:forgetting}, \eqref{eq:w1-w*:cov}, and \eqref{eq:w2-w*:cov}.
Firstly, let's use \eqref{eq:forgetting}, \eqref{eq:w2-w*:cov} and Assumption \ref{assump:commutable} to obtain
\begin{align*}
    \Fcal  
    &= \la \HB^{(1)},\ \Ebb (\wB^{(2)} - \wB^*)(\wB^{(2)} - \wB^*)^\top \ra -  \la \HB^{(1)},\ \Ebb (\wB^{(1)} - \wB^*)(\wB^{(1)} - \wB^*)^\top \ra \\
    &= \la \mu^2 \cdot \big(\HB^{(2)}+\mu\IB \big)^{-2}\HB^{(1)},\ \Ebb (\wB^{(1)} - \wB^*)(\wB^{(1)} - \wB^*)^\top \ra + \frac{\sigma^2}{n} \la \HB^{(1)},\ \big(\HB^{(2)}+\mu\IB \big)^{-2}\HB^{(2)} \ra \\
    &\qquad -  \la \HB^{(1)},\ \Ebb (\wB^{(1)} - \wB^*)(\wB^{(1)} - \wB^*)^\top \ra \\
    &= \frac{\sigma^2}{n} \la \HB^{(1)},\ \big(\HB^{(2)}+\mu\IB \big)^{-2}\HB^{(2)} \ra 
    + \la \mu^2 \cdot \big(\HB^{(2)}+\mu\IB \big)^{-2}\HB^{(1)} - \HB^{(1)},\ \Ebb (\wB^{(1)} - \wB^*)(\wB^{(1)} - \wB^*)^\top \ra \\
    &= \frac{\sigma^2}{n} \la \HB^{(1)},\ \big(\HB^{(2)}+\mu\IB \big)^{-2}\HB^{(2)} \ra 
    - \big\la \big( 2\mu \HB^{(2)} + {\HB^{(2)}}^2 \big) \cdot \big(\HB^{(2)}+\mu\IB \big)^{-2}\cdot \HB^{(1)},\ \Ebb (\wB^{(1)} - \wB^*)(\wB^{(1)} - \wB^*)^\top \big\ra .
\end{align*}
Next, we bring in \eqref{eq:w1-w*:cov} to obtain
\begin{align*}
    \Fcal 
    &= \frac{\sigma^2}{n} \la \HB^{(1)},\ \big(\HB^{(2)}+\mu\IB \big)^{-2}\HB^{(2)} \ra 
    - \frac{\sigma^2}{n} \big\la \big( 2\mu \HB^{(2)} + {\HB^{(2)}}^2 \big) \cdot \big(\HB^{(2)}+\mu\IB \big)^{-2}\cdot \HB^{(1)},\ {\HB^{(1)}}^{-1} \big\ra \\
    &\qquad - \big\la \big( 2\mu \HB^{(2)} + {\HB^{(2)}}^2 \big) \cdot \big(\HB^{(2)}+\mu\IB \big)^{-2}\cdot \HB^{(1)},\ \IB_{\Kbb^c}\cdot \wB^* {\wB^*}^{\top} \cdot \IB_{\Kbb^c}\big\ra  \\
    &= \frac{\sigma^2}{n} \la \HB^{(1)},\ \big(\HB^{(2)}+\mu\IB \big)^{-2}\HB^{(2)} \ra 
    - \frac{\sigma^2}{n} \big\la \big( 2\mu \HB^{(2)} + {\HB^{(2)}}^2 \big) \cdot \big(\HB^{(2)}+\mu\IB \big)^{-2}\cdot \HB^{(1)},\ {\HB^{(1)}}^{-1} \big\ra,
\end{align*}
where the last equality is because $\HB^{(1)}\IB_{\Kbb^c} = 0$ by definition.
We can further simplify $\Fcal$:
\begin{align*}
    \Fcal 
    &=  \frac{\sigma^2}{n}\cdot  \la \HB^{(1)},\ \big(\HB^{(2)}+\mu\IB \big)^{-2}\HB^{(2)} \ra 
     - \frac{\sigma^2}{n} \cdot \big\la \big( 2\mu \IB + {\HB^{(2)}} \big) \cdot \IB_{\Kbb},\ \big(\HB^{(2)}+\mu\IB \big)^{-2} \HB^{(2)} \big\ra \\
     &= \frac{\sigma^2}{n}\cdot  \big\la \HB^{(1)} -\big( 2\mu \IB + {\HB^{(2)}} \big) \cdot \IB_{\Kbb} ,\ \big(\HB^{(2)}+\mu\IB \big)^{-2}\HB^{(2)} \big\ra \\
     &= \frac{\sigma^2}{n}\cdot  \big\la \big( \HB^{(1)} - 2\mu \IB - {\HB^{(2)}} \big) \cdot \IB_{\Kbb} ,\ \big(\HB^{(2)}+\mu\IB \big)^{-2}\HB^{(2)} \big\ra.
\end{align*}
We have completed the proof for the forgetting.

\paragraph{Computing intransigence.}
We next compute intransigence according to \eqref{eq:w1-w*:cov}, \eqref{eq:w2-w*:cov} and \eqref{eq:instransigence} .
We first use \eqref{eq:w2-w*:cov}, \eqref{eq:instransigence}, and Assumption \ref{assump:commutable} to obtain:
\begin{align*}
    \Ical 
    &= \la \HB^{(2)},\ \Ebb (\wB^{(2)} - \wB^*)(\wB^{(2)} - \wB^*)^\top \ra +  \la \HB^{(1)},\ \Ebb (\wB^{(1)} - \wB^*)(\wB^{(1)} - \wB^*)^\top \ra \\
    &= \la \mu^2 \cdot \big(\HB^{(2)}+\mu\IB \big)^{-2}\HB^{(2)},\ \Ebb (\wB^{(1)} - \wB^*)(\wB^{(1)} - \wB^*)^\top \ra + \frac{\sigma^2}{n} \la \HB^{(2)},\ \big(\HB^{(2)}+\mu\IB \big)^{-2}\HB^{(2)} \ra \\
    &\qquad +  \la \HB^{(1)},\ \Ebb (\wB^{(1)} - \wB^*)(\wB^{(1)} - \wB^*)^\top \ra \\
    &= \frac{\sigma^2}{n} \cdot \tr \big\{ {\HB^{(2)}}^2\big(\HB^{(2)}+\mu\IB \big)^{-2} \big\} 
    + \la \mu^2 \cdot \big(\HB^{(2)}+\mu\IB \big)^{-2}\HB^{(2)} + \HB^{(1)},\ \Ebb (\wB^{(1)} - \wB^*)(\wB^{(1)} - \wB^*)^\top \ra.
\end{align*}
Next, let's insert \eqref{eq:w1-w*:cov} into the above, then we obtain
\begin{align*}
    \Ical 
    &= \frac{\sigma^2}{n} \cdot \tr \big\{ {\HB^{(2)}}^2\big(\HB^{(2)}+\mu\IB \big)^{-2} \big\} 
    + \frac{\sigma^2}{n} \cdot \la \mu^2 \cdot \big(\HB^{(2)}+\mu\IB \big)^{-2}\HB^{(2)} + \HB^{(1)},\ {\HB^{(1)}}^{-1} \ra  \\
    &\qquad + \la \mu^2 \cdot \big(\HB^{(2)}+\mu\IB \big)^{-2}\HB^{(2)} + \HB^{(1)},\ \IB_{\Kbb^c}\wB^*{\wB^*}^\top \IB_{\Kbb^c}\ra \\
    &= \frac{\sigma^2}{n} \cdot \tr \big\{ {\HB^{(2)}}^2\big(\HB^{(2)}+\mu\IB \big)^{-2} \big\} 
    + \frac{\sigma^2}{n} \cdot \tr\big\{  \mu^2 \cdot \big(\HB^{(2)}+\mu\IB \big)^{-2}\HB^{(2)}{\HB^{(1)}}^{-1} + \IB_{\Kbb} \big\} \\
    &\qquad + \la \mu^2 \cdot \big(\HB^{(2)}+\mu\IB \big)^{-2}\HB^{(2)}\IB_{\Kbb^c} ,\ \wB^*{\wB^*}^\top \ra \\
    &= \frac{\sigma^2}{n} \cdot\bigg(  \tr \big\{ {\HB^{(2)}}^2\big(\HB^{(2)}+\mu\IB \big)^{-2} \big\} + \tr\big\{  \mu^2 \cdot \big(\HB^{(2)}+\mu\IB \big)^{-2}\HB^{(2)}{\HB^{(1)}}^{-1} \big\}  + \rank\big(\HB^{(1)} \big) \bigg) \\
    &\qquad  + {\wB^*}^\top \cdot \big( \IB_{\Kbb^c} \cdot \mu^2\HB^{(2)} \big(\HB^{(2)}+\mu\IB \big)^{-2} \big) \wB^*.
\end{align*}
We have completed the proof for the intransigence.

\end{proof}

\subsection{Proof of Proposition \ref{prop:joint-learning}}

\begin{proof}[Proof of Proposition \ref{prop:joint-learning}]
Define the joint data by 
\[
\XB := 
\begin{pmatrix}
\XB^{(1)} \\
\XB^{(2)}
\end{pmatrix} \in \Rbb^{2n \times d},\quad 
\yB := 
\begin{pmatrix}
\yB^{(1)} \\
\yB^{(2)}
\end{pmatrix} \in \Rbb^{2n}.
\]
Define 
\[\HB := \HB^{(1)}+\HB^{(2)}.\]
Then  it holds that
\[\XB^\top \XB = \HB^{(1)}+\HB^{(2)} = \HB.\]
Similarly, define the joint label noise by 
\[
\epsilonB := \begin{pmatrix}
\epsilonB^{(1)} \\
\epsilonB^{(2)}
\end{pmatrix} \in \Rbb^{2n},
\]
then 
\[\Ebb \epsilonB = 0,\quad \Ebb \epsilonB\epsilonB^\top = \sigma^2 \IB.\]
By the definition of \eqref{eqn:joint-learning}, we obtain 
\begin{align*}
    \wB_\mathtt{joint} 
    &= \big( {\XB}^\top \XB \big)^{-1} {\XB}^\top \yB \\
    &= \big( {\XB}^\top \XB \big)^{-1} {\XB}^\top \XB \wB^* + \big( {\XB}^\top \XB \big)^{-1} {\XB}^\top \epsilonB,
\end{align*}
which implies that 
\begin{align*}
    \wB_\mathtt{joint}  - \wB^*
    &= \Big( \big( {\XB}^\top \XB \big)^{-1} {\XB}^\top \XB -\IB \Big) \cdot \wB^* + \big( {\XB}^\top \XB \big)^{-1} {\XB}^\top \epsilonB.
\end{align*}
Therefore, it holds that
\begin{align}
    \Ebb (\wB_\mathtt{joint}  - \wB^*)(\wB_\mathtt{joint}  - \wB^*)^\top 
    &= \Big( \big( {\XB}^\top \XB \big)^{-1} {\XB}^\top \XB -\IB \Big) \cdot \wB^*{\wB^*}^\top  \Big( \big( {\XB}^\top \XB \big)^{-1} {\XB}^\top \XB-\IB \Big) + \sigma^2\cdot \big( {\XB}^\top \XB \big)^{-1} \notag \\
    &= ( \HB^{-1} \HB -\IB ) \cdot \wB^*{\wB^*}^\top\cdot   ( \HB^{-1} \HB -\IB )+ \frac{\sigma^2}{n}\cdot \HB^{-1}.\label{eq:wj-w*:cov}
\end{align}
Moreover, the total risk can be reformulated to
\begin{align*}
    \Rcal (\wB) - \min \Rcal 
    &=  \Rcal_1 (\wB) - \min \Rcal_1 +  \Rcal_2 (\wB) - \min \Rcal_2 \\
    &= \la\HB^{(1)},\ (\wB - \wB^*)(\wB - \wB^*)^\top \ra + \la\HB^{(2)},\ (\wB - \wB^*)(\wB - \wB^*)^\top \ra \\
    &= \la\HB,\ (\wB - \wB^*)(\wB - \wB^*)^\top \ra.
\end{align*}
Bringing \eqref{eq:wj-w*:cov} into the above, we obtain 
\begin{align*}
    \Ebb \Rcal (\wB_\joint) - \min \Rcal 
    &= \la\HB,\ \Ebb (\wB_\joint - \wB^*)(\wB_\joint - \wB^*)^\top \ra \\
    &= \la \HB,\ .( \HB^{-1} \HB -\IB ) \cdot \wB^*{\wB^*}^\top\cdot   ( \HB^{-1} \HB -\IB )\ra + \frac{\sigma^2}{n}\cdot \la \HB, \HB^{-1} \ra \\
    &= \frac{\sigma^2}{n}\cdot \rank(\HB) \\
    &= \frac{\sigma^2}{n}\cdot \rank(\HB^{(1)} +\HB^{(2)}  ).
\end{align*}
We have finished the proof.
\end{proof}

\subsection{Proof of Examples}

\begin{proof}[Proof of Example \ref{cor:sequential-example}]
    We examine $\expect{\Rcal(\wB^{(2)}) - \Rcal(\wB^*)}$ by separately examine the order of $\Fcal$ and $\Ical$.
    \begin{enumerate}
        \item By Theorem \ref{thm:regularized-learning} we have
        \begin{align*}
            \Fcal &= \frac{\sigma^2}{n} \big( \tr\big\{ \HB^{(1)}{\HB^{(2)}}^{-1} \big\} - \rank \big\{\HB^{(1)}\HB^{(2)} \big\} \big) \\
            &= \frac{\sigma^2}{n} (\bigO{n^{1-r}} - \bigO{1}) = \bigO{n^{-r}}, \\
            \Ical &= \frac{\sigma^2}{n} (\mathrm{rank} \big\{ \HB^{(2)} + \mathrm{rank} \big\{ \HB^{(2)} = \bigO{n^\inv} \})
        \end{align*}
        As a result, $\expect{\Rcal(\wB^{(2)}) - \Rcal(\wB^*)} = \bigO{n^{-r}}$.
        \item By Theorem \ref{thm:regularized-learning}, we have
        \begin{align*}
            \Fcal &= \frac{\sigma^2}{n} \big( \tr\big\{ \HB^{(1)}{\HB^{(2)}}^{-1} \big\} - \rank \big\{\HB^{(1)}\HB^{(2)} \big\} \big) \\
            &= \frac{\sigma^2}{n} (\bigOmega{n} - \bigO{1}) = \bigOmega{1}, \\
            \Ical &= \frac{\sigma^2}{n} (\mathrm{rank} \big\{ \HB^{(2)} + \mathrm{rank} \big\{ \HB^{(2)} = \bigO{n^\inv} \})
        \end{align*}
        As a result, $\expect{\Rcal(\wB^{(2)}) - \Rcal(\wB^*)} = \bigOmega{1}$.
    \end{enumerate}
\end{proof}

\begin{proof}[Proof of Example \ref{cor:regularized-discussion}]
    We examine $\expect{\Rcal(\wB^{(2)}) - \Rcal(\wB^*)}$ by separately examine the order of $\Fcal$ and $\Ical$.
    \begin{enumerate}
        \item \textbf{OCL:} This is a special case of the second example in Corollary \ref{cor:sequential-example}.
        \item \textbf{$\ell_2$-RCL:}
        By Theorem \ref{thm:regularized-learning} we have
        \begin{align*}
            \Fcal = \frac{\sigma^2}{n} \cdot \dimF, 
            \qquad \Ical = \bias + \frac{\sigma^2}{n} \cdot \dimI,
        \end{align*}
        where
        \begin{equation*}
        \begin{aligned}
            \dimF
            &\le \tr \Big\{ \HB^{(1)} \cdot \HB^{(2)}\cdot \big( \HB^{(2)} + \mu \IB \big)^{-2}\Big\} \\
            &= \bigO{n^{1/3}} + \bigO{1} = \bigO{n^{1/3}}, \\
            \dimI 
            &\le \tr\Big\{ \mu^2 \cdot {\HB^{(1)}}^{-1}\cdot \HB^{(2)}\cdot \big( \HB^{(2)} + \mu\IB \big)^{-2} \Big\} + \rank\big\{ \HB^{(2)} \big\} + \rank\big\{ \HB^{(1)}\big\} \\
            &= \bigO{1}+ \bigO{1}+ \bigO{1} = \bigO{1},\\
            \bias &\le \mu \cdot \| \wB^*\|_2^2 = \bigO{n^{-2/3}}.
        \end{aligned}
        \end{equation*}
        As a result, $\Fcal = \bigO{n^{-2/3}}$, $\Ical = \bigO{n^{-2/3}}$, and $\expect{\Rcal(\wB^{(2)}) - \Rcal(\wB^*)} = \bigO{n^{-2/3}}$.
    \end{enumerate}
\end{proof}

\begin{proof}[Proof of Example \ref{cor:regularized-limitation}]
    Similar to the above, we separately examine the order of $\Fcal$ and $\Ical$.
    By Theorem \ref{thm:regularized-learning}, the forgetting satisfies
    \begin{align*}
        \Fcal= {2\sigma^2 \over n} {\mu^2 \over (1/n + \mu)^2} + {\sigma^2 \over n} (n + n^\inv) {(1/n)^2 \over (1/n + \mu)^2} - {\sigma^2 \over n} \mathrm{rank}(\HB^{(1)})
    \end{align*}
    so $\Fcal=\bigOmega{1}$ for $\mu = \bigO{n^\inv}$, and $\Fcal=\bigO{n^\inv}$ for $\mu = \bigTheta{n^{-0.5}}$;
    and the intransigence satisfies
    \begin{align*}
        \Ical &\leq \mu \|\wB^*\|_2^2 + {\sigma^2 \over n} (n + n^\inv) {\mu^2 \over (1/n + \mu)^2} + {2\sigma^2 \over n}  {(1/n)^2 \over (1/n + \mu)^2} + {\sigma^2 \over n} \mathrm{rank}(\HB^{(1)}) \\
        \Ical &\geq {\sigma^2 \over n} (n + n^\inv) {\mu^2 \over (1/n + \mu)^2} + {2\sigma^2 \over n}  {(1/n)^2 \over (1/n + \mu)^2} + {\sigma^2 \over n} \mathrm{rank}(\HB^{(1)})
    \end{align*}
    so $\Ical=\bigOmega{1}$ for $\mu = \bigOmega{n^\inv}$, and $\Ical=\bigO{n^\inv}$ for $\mu = \bigTheta{n^{-1.5}}$.

    We note that for every $\mu>0$,
    \begin{align*}
        \expect{\Rcal(\wB^{(2)}) - \Rcal(\wB^*)} &= \Fcal + \Ical \\
        &\geq {\sigma^2 \over n} (n + n^\inv) {\mu^2 \over (1/n + \mu)^2} + {\sigma^2 \over n} (n + n^\inv) {(1/n)^2 \over (1/n + \mu)^2} \\
        &\geq \sigma^2 \left( {\mu^2 \over (1/n + \mu)^2} +  {(1/n)^2 \over (1/n + \mu)^2} \right) = \bigOmega{1} .
    \end{align*}
\end{proof}

\section{Extension of Section \ref{sec:risk-CL}}
The main paper discusses the continual learning implications when there exist certain assumptions, including Assumptions \ref{assump:shared-optimal} and \ref{assump:commutable}.
In this section we would like to discuss the case without these assumptions.

\subsection{Risk Bounds For Continual Learning without Assumption \ref{assump:shared-optimal}}

\begin{theorem}[A risk bound for $\ell_2$-RCL/OCL without Assumption \ref{assump:shared-optimal}]\label{thm:risk-non-shared-optimal}
    Suppose that Assumptions \ref{assump:fixed-design} , \ref{assump:noise} and \ref{assump:commutable} hold.
    Denote ${\wB^{(1)}}^*, {\wB^{(2)}}^*\in\Rbb^d$ such that 
    \[{\wB^{(1)}}^* \in \arg\min \Rcal_1(\wB),\quad {\wB^{(2)}}^* \in \arg\min \Rcal_2(\wB).\]
    Then for the $\ell_2$-RCL output \eqref{eqn:regularized-learning}, it holds that 
    \begin{equation*}
        \Ebb\big[ \Rcal\big( \wB^{(2)} \big) \big]  - \min \Rcal (\cdot) = \Ebb \big[ \Fcal\big(\wB^{(2)}, \wB^{(1)}\big) \big] + \Ebb \big[ \Ical\big(\wB^{(2)}, \wB^{(1)}\big) \big] .
    \end{equation*}
    Moreover, it holds that
    \begin{align*}
        \Ebb \big[ \Fcal\big(\wB^{(2)}, \wB^{(1)}\big) \big]
        &= \biasF + \frac{\sigma^2}{n} \cdot \dimF, \\
        \Ebb \big[ \Ical\big(\wB^{(2)}, \wB^{(1)}\big) \big] 
        &= \biasI + \frac{\sigma^2}{n} \cdot \dimI,
    \end{align*}
    where
    \begin{equation}\label{eqn:bias-non-shared-optimal}
    \begin{aligned}
        \biasF & := ({\wB^{(1)}}^* - {\wB^{(2)}}^*)^\top \cdot \HB^{(1)} \cdot {\HB^{(2)}}^2 \cdot \big(  {\HB^{(2)} +  \mu\IB}\big)^{-2} \cdot ({\wB^{(1)}}^* - {\wB^{(2)}}^*), \\
        \biasI & := (\IB_\Kbb {\wB^{(1)}}^* - {\wB^{(2)}}^*)^\top \cdot \mu^2 \cdot \HB^{(2)} \cdot \big(  {\HB^{(2)} +  \mu\IB}\big)^{-2} \cdot (\IB_\Kbb {\wB^{(1)}}^* - {\wB^{(2)}}^*),
    \end{aligned}
    \end{equation}
    for index sets $\Kbb := \big\{ i: \lambda_i^{(1)} >  0 \big\}$, and
    \begin{equation}\label{eqn:eff-dim-non-shared-optimal}
    \begin{aligned}
        \dimF & :=   \tr\Big\{\IB_{\Kbb}\cdot \big( \HB^{(1)}-\HB^{(2)}-2\mu\IB\big) \cdot \HB^{(2)}\cdot \big( \HB^{(2)} + \mu \IB \big)^{-2} \Big\}, \\
        \dimI & := \tr\Big\{ \big( \mu^2 \cdot {\HB^{(1)}}^{-1} + \HB^{(2)}\big)\cdot \HB^{(2)}\cdot \big( \HB^{(2)} + \mu\IB \big)^{-2} \Big\} + \rank\big\{ \HB^{(1)} \big\}, 
    \end{aligned}
    \end{equation}
    are the same as in Theorem \ref{thm:regularized-learning}.
\end{theorem}

This shows that Assumption \ref{assump:shared-optimal} only has an effect on the bias term of the bound.
The implication of this result is twofold.

Firstly, in terms of the OCL algorithm, to achieve an $o(1)$ CL excess risk, it suffices that 
\begin{equation}\label{eq:ocl:task-similarity:non-shared-optimal}
    \tr\big\{ \HB^{(1)} {\HB^{(2)}}^{-1} \big\} \leq o(n) ,\quad ({\wB^{(1)}}^* - {\wB^{(2)}}^*)^\top \cdot \HB^{(1)} \cdot {\HB^{(2)}}^2 \cdot \big( \HB^{(2)}\big)^{-2} \cdot ({\wB^{(1)}}^* - {\wB^{(2)}}^*) = o(1).
\end{equation}
This can count as another similarity measurement of OCL in the non-shared-optimal case, and is tougher than the condition in \eqref{eq:ocl:task-similarity}.
A necessary condition for this to happen is that ${\wB^{(1)}}^* - {\wB^{(2)}}^*$ has $o(1)$ entries for all indices $i$ that satisfy $ \lambda_i^{(1)} \geq c, \lambda_i^{(2)} > 0$ for a constant $c$, which is not the usual case.

Secondly we explain the effect on the $\ell_2$-RCL algorithm.
According to \eqref{eqn:bias-non-shared-optimal} and \eqref{eqn:eff-dim-non-shared-optimal}, the regularizer coefficient $\mu$ controls a trade-off between forgetting and intransigence, in similar terms to the non-shared-optimal case that increasing $\mu$ reduces forgetting ($\biasF$ and $\dimF$) while produces intransigence ($\biasI$ and $\dimI$), and vice versa.
Moreover, comparing to the OCL algorithm and the scenario with shared optimal, for the $\ell_2$-RCL to achieve an $o(1)$ CL excess risk without Assumption \ref{assump:shared-optimal}, it suffices that $\biasF, \biasI = o(1)$ and $\dimF, \dimI = o(n)$.
Consider the upper bound of $\dimF, \dimI$ in \eqref{eqn:eff-dim-upper-bound}, it suffices to require that 
\begin{equation}\label{eq:rcl:task-similarity:non-shared-optimal}
\begin{gathered}
    \tr \Big\{ \HB^{(1)} \cdot  \big( \HB^{(2)} + \mu \IB \big)^{-1}\Big\} \le o(n), \quad 
    \mu \cdot \tr\Big\{ {\HB^{(1)}}^{-1} \Big\} \le o(n), \quad\rank(\HB^{(1)}) + \rank(\HB^{(2)}) \leq o(n), \\
    ({\wB^{(1)}}^* - {\wB^{(2)}}^*)^\top \cdot \HB^{(1)} \cdot {\HB^{(2)}}^2 \cdot \big(  {\HB^{(2)} +  \mu\IB}\big)^{-2} \cdot ({\wB^{(1)}}^* - {\wB^{(2)}}^*) \leq o(1), \\
    (\IB_\Kbb {\wB^{(1)}}^* - {\wB^{(2)}}^*)^\top \cdot \mu^2 \cdot \HB^{(2)} \cdot \big(  {\HB^{(2)} +  \mu\IB}\big)^{-2} \cdot (\IB_\Kbb {\wB^{(1)}}^* - {\wB^{(2)}}^*)^\top \leq o(1),
\end{gathered}
\end{equation}
for some $\mu > 0$. 
The OCL condition in \eqref{eq:ocl:task-similarity:non-shared-optimal} is a special case of this condition when ${\wB^{(1)}}^* = {\wB^{(2)}}^*$.

\subsection{Risk Bounds For Continual Learning without Assumption \ref{assump:commutable}}

\begin{theorem}[A risk bound for $\ell_2$-RCL/OCL without Assumption \ref{assump:commutable}]\label{thm:risk-non-commutable}
    Suppose that Assumptions \ref{assump:fixed-design} to \ref{assump:noise} hold.
    Then for the $\ell_2$-RCL output \eqref{eqn:regularized-learning}, it holds that 
    \begin{equation*}
        \Ebb\big[ \Rcal\big( \wB^{(2)} \big) \big]  - \min \Rcal (\cdot) = \Ebb \big[ \Fcal\big(\wB^{(2)}, \wB^{(1)}\big) \big] + \Ebb \big[ \Ical\big(\wB^{(2)}, \wB^{(1)}\big) \big] .
    \end{equation*}
    Moreover, it holds that
    \begin{align*}
        \Ebb \big[ \Fcal\big(\wB^{(2)}, \wB^{(1)}\big) \big]
        &= \biasF + \frac{\sigma^2}{n} \cdot \dimF, \\
        \Ebb \big[ \Ical\big(\wB^{(2)}, \wB^{(1)}\big) \big] 
        &= \biasI + \frac{\sigma^2}{n} \cdot \dimI,
    \end{align*}
    where
    \begin{equation}
    \begin{aligned}
        \dimF & := \tr\Big\{ \big( \mu^2 \cdot {\HB^{(1)}}^{-1} + \HB^{(2)}\big)\cdot \big( \HB^{(2)} + \mu\IB \big)^{-1} \cdot \HB^{(1)}\cdot \big( \HB^{(2)} + \mu\IB \big)^{-1} \Big\} - \rank\big\{ \HB^{(1)} \big\}, \\
        \dimI & := \tr\Big\{ \big( \mu^2 \cdot {\HB^{(1)}}^{-1} + \HB^{(2)}\big)\cdot \big( \HB^{(2)} + \mu\IB \big)^{-1} \cdot \HB^{(2)}\cdot \big( \HB^{(2)} + \mu\IB \big)^{-1} \Big\} + \rank\big\{ \HB^{(1)} \big\}, \\
        \biasF & := (\wB^*)^\top \cdot \IB_{\Kbb^c} \cdot \mu^2 \cdot \big(  {\HB^{(2)} +  \mu\IB}\big)^{-1} \cdot \HB^{(1)} \cdot \big(  {\HB^{(2)} +  \mu\IB}\big)^{-1} \cdot \wB^*, \\
        \biasI & := (\wB^*)^\top \cdot \IB_{\Kbb^c} \cdot \mu^2 \cdot \big(  {\HB^{(2)} +  \mu\IB}\big)^{-1} \cdot \HB^{(2)} \cdot \big(  {\HB^{(2)} +  \mu\IB}\big)^{-1} \cdot \wB^*, 
    \end{aligned}
    \end{equation}
    for index sets $\Kbb := \big\{ i: \lambda_i^{(1)} >  0 \big\}$.
\end{theorem}

This shows that by dropping Assumption \ref{assump:commutable}, to achieve an $o(1)$ CL excess risk, it suffices that
\begin{equation}\label{eq:rcl:task-similarity:non-commutable}
\begin{gathered}
    \tr\Big\{ \big( \mu^2 \cdot {\HB^{(1)}}^{-1} + \HB^{(2)}\big)\cdot \big( \HB^{(2)} + \mu\IB \big)^{-1} \cdot \big( \HB^{(1)} + \HB^{(2)} \big)  \cdot \big( \HB^{(2)} + \mu\IB \big)^{-1} \Big\} \leq o(n), \\
    (\wB^*)^\top \cdot \IB_{\Kbb^c} \cdot \mu^2 \cdot \big(  {\HB^{(2)} +  \mu\IB}\big)^{-1} \cdot \HB^{(1)} \cdot \big(  {\HB^{(2)} +  \mu\IB}\big)^{-1} \cdot \wB^* \leq o(1), \quad \mu \leq o(1).
\end{gathered}
\end{equation}

We note that \eqref{eq:rcl:task-similarity:non-commutable} and Theorem \ref{thm:risk-non-commutable} are extensions of the main result in \eqref{eqn:eff-dim-upper-bound} and Theorem \ref{thm:regularized-learning}.

\subsection{Proof of Theorem \ref{thm:risk-non-shared-optimal}}

\begin{proof}[Proof of Theorem \ref{thm:risk-non-shared-optimal}]
    The proof can be done by direct computation and is similar to the proof of Theorem \ref{thm:regularized-learning}.

\paragraph{Risk decomposition.}
According to the risk definition and the assumptions on the noise, we have 
\begin{align}
    \Rcal_1(\wB) &= \frac{1}{n}\Ebb\big\| \XB^{(1)}\wB - \yB^{(1)} \big\|_2^2 \notag \\  
    &= \frac{1}{n}\Ebb\big\| \XB^{(1)}\wB - \XB^{(1)}\wB^* - \epsilonB^{(1)} \big\|_2^2 \notag \\ 
    &= (\wB - \wB^*)^\top \HB^{(1)} (\wB - {\wB^{(1)}}^*) + \sigma^2 \notag \\
    &= \la\HB^{(1)},\ (\wB - {\wB^{(1)}}^*)(\wB - {\wB^{(1)}}^*)^\top \ra + \sigma^2. \label{eq:risk1:non-shared-optimal}
\end{align}
Similarly, the risk for the second task is 
\begin{align}
    \Rcal_2 (\wB) =  \la\HB^{(2)},\ (\wB - {\wB^{(2)}}^*)(\wB - {\wB^{(1)}}^*)^\top \ra + \sigma^2. \label{eq:risk2:non-shared-optimal}
\end{align}

Based on \eqref{eq:risk1:non-shared-optimal} and \eqref{eq:risk2:non-shared-optimal}, we can compute the forgetting and intransigence as follows: 
\begin{align}
    \Fcal &:= \Ebb [ \Rcal_1 (\wB^{(2)}) - \Rcal_1 (\wB^{(1)}) ] \notag \\ 
    &= \la \HB^{(1)},\ \Ebb (\wB^{(2)} - {\wB^{(1)}}^*)(\wB^{(2)} - {\wB^{(1)}}^*)^\top \ra -  \la \HB^{(1)},\ \Ebb (\wB^{(1)} - {\wB^{(1)}}^*)(\wB^{(1)} - {\wB^{(1)}}^*)^\top \ra \label{eq:forgetting:non-shared-optimal},
\end{align}
and 
\begin{align}
    \Ical &:= \Ebb [ \Rcal_2 (\wB^{(2)})] - \min \Rcal_2 + \Ebb [\Rcal_1 (\wB^{(1)}) ] - \min \Rcal_1 \notag  \\ 
    &= \la \HB^{(2)},\ \Ebb (\wB^{(2)} - {\wB^{(2)}}^*)(\wB^{(2)} - {\wB^{(2)}}^*)^\top \ra +  \la \HB^{(1)},\ \Ebb (\wB^{(1)} - {\wB^{(1)}}^*)(\wB^{(1)} - {\wB^{(1)}}^*)^\top \ra \label{eq:instransigence:non-shared-optimal}.
\end{align}

\paragraph{Computing $\wB^{(1)}$.}
We first compute $\wB^{(1)}$. By definition, we have 
\begin{align*}
    \wB^{(1)} &= \big((\XB^{(1)})^\top\XB^{(1)}  \big)^{-1} (\XB^{(1)})^\top \yB \\
    &= \big((\XB^{(1)})^\top\XB^{(1)}  \big)^{-1} (\XB^{(1)})^\top \XB^{(1)}{\wB^{(1)}}^* + \big((\XB^{(1)})^\top\XB^{(1)}  \big)^{-1} (\XB^{(1)})^\top\epsilonB^{(1)} \\
    &= \IB_{\Kbb} \cdot \wB^* + \frac{1}{n} { \HB^{(1)} }^{-1} {\XB^{(1)}}^\top \epsilonB^{(1)}.
\end{align*}
Therefore
\begin{align}
     \wB^{(1)} - {\wB^{(1)}}^* 
     = - \IB_{\Kbb^c} \cdot {\wB^{(1)}}^* + \frac{1}{n} { \HB^{(1)} }^{-1} {\XB^{(1)}}^\top \epsilonB^{(1)}.\notag 
\end{align}
Now noticing that 
\[\Ebb\epsilonB^{(1)} = 0,\quad 
\Ebb \epsilonB^{(1)}{\epsilonB^{(1)}}^\top = \sigma^2 \IB,
\]
so the covariance of $\wB^{(1)} - {\wB^{(1)}}^* $ is
\begin{align}
    \Ebb (\wB^{(1)} - {\wB^{(1)}}^*)(\wB^{(1)} - {\wB^{(1)}}^* )^\top 
    &= \IB_{\Kbb^c} \cdot {\wB^{(1)}}^*{{\wB^{(1)}}^*}^\top \cdot \IB_{\Kbb^c} + \frac{\sigma^2 }{n^2} { \HB^{(1)} }^{-1} {\XB^{(1)}}^\top \XB^{(1)}{ \HB^{(1)} }^{-1} \notag \\
    &= \IB_{\Kbb^c} \cdot {\wB^{(1)}}^*{{\wB^{(1)}}^*}^\top \cdot \IB_{\Kbb^c} + \frac{\sigma^2}{n}\cdot { \HB^{(1)} }^{-1}. 
\end{align}
and that
\begin{align} \label{eq:r1w1:non-shared-optimal}
    \la \HB^{(1)},\ \Ebb (\wB^{(1)} - {\wB^{(1)}}^*)(\wB^{(1)} - {\wB^{(1)}}^*)^\top \ra = \frac{\sigma^2}{n} \rank(\HB^{(1)}) 
\end{align}
since by definition  $\HB^{(1)}\IB_{\Kbb^c} = 0$.

\paragraph{Computing $\wB^{(2)}$.}
Then we compute $\wB^{(2)}$, which is the solution of \eqref{eqn:regularized-learning}. By the first-order optimality condition, we have 
\begin{align*}
    \frac{1}{n} {\XB^{(2)}}^\top \big( \XB^{(2)} \wB^{(2)} - \yB^{(2)} \big) + \mu (\wB^{(2)} - \wB^{(1)}) = 0,
\end{align*}
which implies 
\begin{align*}
    \wB^{(2)} 
    &= \bigg( \frac{1}{n}{\XB^{(2)}}^\top \XB^{(2)}  + \mu \IB \bigg)^{-1} \bigg( \frac{1}{n}{\XB^{(2)}}^\top \yB^{(2)} + \mu \wB^{(1)} \bigg)\\
    &= \bigg( \frac{1}{n}{\XB^{(2)}}^\top \XB^{(2)}  + \mu \IB \bigg)^{-1} \bigg( \frac{1}{n}{\XB^{(2)}}^\top\XB^{(2)} {\wB^{(2)}}^* + \frac{1}{n}{\XB^{(2)}}^\top\epsilonB^{(2)}  + \mu \wB^{(1)} \bigg) \\
    &= \big( \HB^{(2)} + \mu \IB \big)^{-1}\bigg(  \HB^{(2)}{\wB^{(2)}}^* + \frac{1}{n}{\XB^{(2)}}^\top\epsilonB^{(2)}  + \mu \wB^{(1)}  \bigg).
\end{align*}
Then we have 
\begin{align*}
    \wB^{(2)}  - {\wB^{(2)}}^*
    &= \big( \HB^{(2)} + \mu \IB \big)^{-1}\bigg(  \HB^{(2)}{\wB^{(2)}}^* + \frac{1}{n}{\XB^{(2)}}^\top\epsilonB^{(2)}  + \mu \wB^{(1)} -  \big( \HB^{(2)} + \mu \IB \big){\wB^{(2)}}^*\bigg) \\
    &= \big( \HB^{(2)} + \mu \IB \big)^{-1}\bigg(  \mu (\wB^{(1)}-{\wB^{(2)}}^*)  + \frac{1}{n}{\XB^{(2)}}^\top\epsilonB^{(2)} \bigg) \\
    &= \big( \HB^{(2)} + \mu \IB \big)^{-1}\bigg(  \mu (\IB_{\Kbb} \cdot {\wB^{(1)}}^* -{\wB^{(2)}}^*) + \frac{1}{n} \mu{ \HB^{(1)} }^{-1} {\XB^{(1)}}^\top \epsilonB^{(1)}  + \frac{1}{n}{\XB^{(2)}}^\top\epsilonB^{(2)} \bigg) ,
\end{align*}
and
\begin{align*}
    \wB^{(2)}  - {\wB^{(1)}}^*
    &= \big( \HB^{(2)} + \mu \IB \big)^{-1}\bigg(  \mu (\wB^{(1)}-{\wB^{(2)}}^*) - \big( \HB^{(2)} + \mu \IB \big) ({\wB^{(1)}}^*-{\wB^{(2)}}^*) + \frac{1}{n}{\XB^{(2)}}^\top\epsilonB^{(2)} \bigg) \\
    &= \big( \HB^{(2)} + \mu \IB \big)^{-1}\bigg(  \mu (\wB^{(1)}-{\wB^{(1)}}^*) - \HB^{(2)}({\wB^{(1)}}^*-{\wB^{(2)}}^*) + \frac{1}{n}{\XB^{(2)}}^\top\epsilonB^{(2)} \bigg) \\
    &= \big( \HB^{(2)} + \mu \IB \big)^{-1}\bigg(  -\mu \IB_{\Kbb^c} \cdot {\wB^{(1)}}^* - \HB^{(2)}({\wB^{(1)}}^*-{\wB^{(2)}}^*) + \frac{1}{n} \mu{ \HB^{(1)} }^{-1} {\XB^{(1)}}^\top \epsilonB^{(1)} + \frac{1}{n}{\XB^{(2)}}^\top\epsilonB^{(2)} \bigg) .
\end{align*}
Similarly, noticing that 
\[\Ebb\epsilonB^{(1)} = 0,\quad 
\Ebb \epsilonB^{(1)}{\epsilonB^{(1)}}^\top = \sigma^2 \IB,\quad\Ebb\epsilonB^{(2)} = 0,\quad 
\Ebb \epsilonB^{(2)}{\epsilonB^{(2)}}^\top = \sigma^2 \IB,
\]
so 
\begin{align}
    \la \HB^{(1)} , \Ebb (\wB^{(2)}  - {\wB^{(1)}}^*)(\wB^{(2)}  - {\wB^{(1)}}^*)^\top \ra
    &= ({\wB^{(1)}}^* - {\wB^{(2)}}^*)^\top \cdot \HB^{(1)} \cdot {\HB^{(2)}}^2 \cdot \big(  {\HB^{(2)} +  \mu\IB}\big)^{-2} \cdot ({\wB^{(1)}}^* - {\wB^{(2)}}^*) \notag \\
    &\qquad + \frac{\sigma^2}{n} \mu^2 \la \HB^{(2)}, \big( \HB^{(2)} + \mu \IB \big)^{-1} \cdot {\HB^{(1)}}^\inv \cdot \big( \HB^{(2)} + \mu \IB \big)^{-1} \ra \notag \\
    &\qquad + \frac{\sigma^2}{n} \la \HB^{(2)}, \big( \HB^{(2)} + \mu \IB \big)^{-1} \cdot \HB^{(2)} \cdot \big( \HB^{(2)} + \mu \IB \big)^{-1} \ra \notag \\
    &= \biasF + \frac{\sigma^2}{n} \big(\dimF + \rank(\HB^{(1)}) \big), \label{eq:r1w2:non-shared-optimal}
\end{align}
where the first line is because  by definition  $\HB^{(1)}\IB_{\Kbb^c} = 0$, and
\begin{align}
    \la \HB^{(2)} , \Ebb (\wB^{(2)}  - {\wB^{(2)}}^*)(\wB^{(2)}  - {\wB^{(2)}}^*)^\top \ra
    &= (\IB_\Kbb {\wB^{(1)}}^* - {\wB^{(2)}}^*)^\top \cdot \mu^2 \cdot \HB^{(2)} \cdot \big(  {\HB^{(2)} +  \mu\IB}\big)^{-2} \cdot (\IB_\Kbb {\wB^{(1)}}^* - {\wB^{(2)}}^*) \notag \\
    &\qquad + \frac{\sigma^2}{n} \mu^2 \la \HB^{(2)}, \big( \HB^{(2)} + \mu \IB \big)^{-1} \cdot {\HB^{(1)}}^\inv \cdot \big( \HB^{(2)} + \mu \IB \big)^{-1} \ra \notag \\
    &\qquad + \frac{\sigma^2}{n} \la \HB^{(2)}, \big( \HB^{(2)} + \mu \IB \big)^{-1} \cdot \HB^{(2)} \cdot \big( \HB^{(2)} + \mu \IB \big)^{-1} \ra \notag \\
    &= \biasI + \frac{\sigma^2}{n} \big(\dimI - \rank(\HB^{(1)}) \big). \label{eq:r2w2:non-shared-optimal}
\end{align}
Combine \eqref{eq:r1w1:non-shared-optimal}, \eqref{eq:r1w2:non-shared-optimal}, \eqref{eq:r2w2:non-shared-optimal} and we are done with the proof of forgetting and intransigence.
\end{proof}

\subsection{Proof of Theorem \ref{thm:risk-non-commutable}}

\begin{proof}[Proof of Theorem \ref{thm:risk-non-commutable}]
    The proof can be done by direct computation and is similar to the proof of Theorem \ref{thm:regularized-learning}.

We can adopt \eqref{eq:w1-w*:cov}, \eqref{eq:w2-w*:cov}, \eqref{eq:risk1}, \eqref{eq:risk2} from the proof of Theorem \ref{thm:regularized-learning} that
\begin{align*}
    \Ebb (\wB^{(1)} - \wB^*)(\wB^{(1)} - \wB^* )^\top 
    &= \IB_{\Kbb^c} \cdot \wB^*{\wB^*}^\top \cdot \IB_{\Kbb^c} + \frac{\sigma^2}{n}\cdot { \HB^{(1)} }^{-1}, \\
    \Ebb (\wB^{(2)}  - \wB^*)(\wB^{(2)}  - \wB^*)^\top
    &= \mu^2\cdot \big( \HB^{(2)} + \mu \IB \big)^{-1} \cdot \Ebb (\wB^{(1)}  - \wB^*)(\wB^{(1)}  - \wB^*)^\top \cdot \big( \HB^{(2)} + \mu \IB \big)^{-1} \notag \\
    &\qquad + \frac{\sigma^2}{n} \big( \HB^{(2)} + \mu \IB \big)^{-1} \cdot \HB^{(2)} \cdot \big( \HB^{(2)} + \mu \IB \big)^{-1}, \\
    \Rcal_1 (\wB) &=  \la\HB^{(1)},\ (\wB - \wB^*)(\wB - \wB^*)^\top \ra + \sigma^2, \\
    \Rcal_2 (\wB) &=  \la\HB^{(2)},\ (\wB - \wB^*)(\wB - \wB^*)^\top \ra + \sigma^2. 
\end{align*}
As a result,
\begin{align} \label{eq:r1w1:non-commutable}
    \la \HB^{(1)},\ \Ebb (\wB^{(1)} - \wB^*)(\wB^{(1)} - \wB^*)^\top \ra = \frac{\sigma^2}{n} \rank(\HB^{(1)}) 
\end{align}
since by definition  $\HB^{(1)}\IB_{\Kbb^c} = 0$,
\begin{align}
    \la \HB^{(1)},\ \Ebb (\wB^{(2)} - \wB^*)(\wB^{(2)} - \wB^*)^\top \ra 
    &= \mu^2\cdot \la \HB^{(1)},\  \big( \HB^{(2)} + \mu \IB \big)^{-1} \cdot \IB_{\Kbb^c} \cdot \wB^*{\wB^*}^\top \cdot \IB_{\Kbb^c} \cdot \big( \HB^{(2)} + \mu \IB \big)^{-1} \ra \notag \\
    &\qquad + \frac{\sigma^2}{n} \mu^2 \cdot\la \HB^{(1)},\  \big( \HB^{(2)} + \mu \IB \big)^{-1} \cdot { \HB^{(1)} }^{-1} \cdot \big( \HB^{(2)} + \mu \IB \big)^{-1} \ra \notag \\
    &\qquad + \frac{\sigma^2}{n} \la \HB^{(1)},\  \big( \HB^{(2)} + \mu \IB \big)^{-1} \cdot \HB^{(2)} \cdot \big( \HB^{(2)} + \mu \IB \big)^{-1} \ra \notag \\
    &= \biasF + \frac{\sigma^2}{n} \big(\dimF + \rank(\HB^{(1)}) \big), \label{eq:r1w2:non-commutable}
\end{align}
and that
\begin{align}
    \la \HB^{(2)},\ \Ebb (\wB^{(2)} - \wB^*)(\wB^{(2)} - \wB^*)^\top \ra 
    &= \mu^2\cdot \la \HB^{(2)},\  \big( \HB^{(2)} + \mu \IB \big)^{-1} \cdot \IB_{\Kbb^c} \cdot \wB^*{\wB^*}^\top \cdot \IB_{\Kbb^c} \cdot \big( \HB^{(2)} + \mu \IB \big)^{-1} \ra \notag \\
    &\qquad + \frac{\sigma^2}{n} \mu^2 \cdot\la \HB^{(2)},\  \big( \HB^{(2)} + \mu \IB \big)^{-1} \cdot { \HB^{(1)} }^{-1} \cdot \big( \HB^{(2)} + \mu \IB \big)^{-1} \ra \notag \\
    &\qquad + \frac{\sigma^2}{n} \la \HB^{(2)},\  \big( \HB^{(2)} + \mu \IB \big)^{-1} \cdot \HB^{(2)} \cdot \big( \HB^{(2)} + \mu \IB \big)^{-1} \ra \notag \\
    &= \biasI + \frac{\sigma^2}{n} \big(\dimI - \rank(\HB^{(1)}) \big). \label{eq:r2w2:non-commutable}
\end{align}
Combine \eqref{eq:r1w1:non-commutable}, \eqref{eq:r1w2:non-commutable}, \eqref{eq:r2w2:non-commutable} and we are done with the proof of forgetting and intransigence.
\end{proof}


\end{document}